\documentclass{article}
\usepackage{geometry}

\usepackage[utf8]{inputenc} 
\usepackage[T1]{fontenc}    
\usepackage[pdflang={en-US}]{hyperref}       
\usepackage{url}            
\usepackage{booktabs}       
\usepackage{amsfonts}       
\usepackage{nicefrac}       
\usepackage{xcolor}

\usepackage{microtype}
\usepackage{graphicx}
\usepackage{booktabs} 
\usepackage{hyperref}
\usepackage{natbib}

\usepackage{algorithm}
\usepackage{algorithmic}

\usepackage{amsmath}
\usepackage{amssymb}
\usepackage{mathtools}
\usepackage{amsthm}
\usepackage{subfig}
\usepackage{bbm}
\usepackage{dsfont}
\usepackage{authblk}

\usepackage[capitalize,noabbrev]{cleveref}

\theoremstyle{plain}
\newtheorem{theorem}{Theorem}[section]

\newtheorem{lemma}[theorem]{Lemma}

\theoremstyle{definition}

\theoremstyle{remark}
\newtheorem{remark}[theorem]{Remark}

\usepackage[textsize=tiny]{todonotes}

\newcommand{\E}{\mathbb{E}}
\newcommand{\R}{\mathbb{R}}
\newcommand{\N}{\mathbb{N}}

\newcommand{\1}{\mathds{1}}

\newcommand{\cS}{\mathcal{S}}

\newcommand{\ip}[2] {\langle #1, #2 \rangle }

\newcommand{\cV}{\mathcal{V}}
\newcommand{\cE}{\mathcal{E}}

\newcommand{\cA}{\mathcal{A}}

\newcommand{\cZ}{\mathcal{Z}}
\newcommand{\cC}{\mathcal{C}}
\newcommand{\cL}{\mathcal{L}}
\newcommand{\cO}{\mathcal{O}}

\newcommand{\cN}{\mathcal{N}}
\newcommand{\cB}{\mathcal{B}}

\newcommand{\bfx}{\mathbf{x}}
\newcommand{\bfy}{\mathbf{y}}
\newcommand{\bfW}{\mathbf{W}}
\newcommand{\bfZ}{\mathbf{Z}}
\newcommand{\bfX}{\mathbf{X}}
\newcommand{\bfY}{\mathbf{Y}}
\newcommand{\bfV}{\mathbf{V}}
\newcommand{\bfv}{\mathbf{v}}
\newcommand{\bfu}{\mathbf{u}}
\newcommand{\bfz}{\mathbf{z}}

\newcommand{\bfmu}{\boldsymbol{\mu}}
\newcommand{\bfupsilon}{\boldsymbol{\upsilon}}

\DeclareMathOperator*{\argmin}{arg\,min}

\usepackage{hyperref}  
\hypersetup{
    colorlinks=true,
    linkcolor=blue,
    citecolor =blue,
    filecolor=magenta,      
    urlcolor=magenta,
}

\title{Distributed Linear Bandits under Communication Constraints}
\author{Sudeep Salgia}
\author{Qing Zhao}
\affil{School of Electrical \& Computer Engineering, Cornell University, Ithaca, NY, \emph{\{ss3827,qz16\}@cornell.edu} }
\date{}

\begin{document}

\maketitle

\begin{abstract}  
We consider distributed linear bandits where $M$ agents learn collaboratively to minimize the overall cumulative regret incurred by all agents. Information exchange is facilitated by a central server, and both the uplink and downlink communications are carried over channels with fixed capacity, which limits the amount of information that can be transmitted in each use of the channels. We investigate the regret-communication trade-off by (i) establishing information-theoretic lower bounds on the required communications (in terms of bits) for achieving a sublinear regret order; (ii) developing an efficient algorithm that achieves the minimum sublinear regret order offered by centralized learning using the minimum order of communications dictated by the information-theoretic lower bounds. For sparse linear bandits, we show a variant of the proposed algorithm offers better regret-communication trade-off by leveraging the sparsity of the problem. 
\end{abstract}

\section{Introduction}

The tension between learning efficiency and communication cost is evident in many distributed learning problems. If distributed agents can share all their locally obtained information and fully coordinate their actions, the problem effectively reduces to a centralized problem, and the greatest learning efficiency defined by the centralized counterpart is trivially achieved at the price of high communication cost. \\

What is the minimum amount (in terms of bits) of communications needed to achieve the learning efficiency offered by centralized learning? How one might design a distributed learning algorithm that operates at such an optimal point in the communication-learning efficiency trade-off curve? These fundamental questions have not been adequately addressed in the literature. 

\subsection{Main Results}

In this paper, we address the above questions within the scope of distributed linear bandits. We consider a system of $M$ distributed agents whose actions generate random rewards governed by a common unknown mean $\theta^* \in \R^d$. The agents aim to optimize their actions over time to minimize the overall cumulative regret incurred by all agents over a horizon of length $T$. Communications across agents are facilitated by a central server. To quantify the communication cost to the bit level, we assume that both the uplink and downlink channels have a capacity of $R$ bits per channel use. \\

Our main results are twofold. First, we establish an information-theoretic lower bound on the required communications for achieving a sublinear regret order. Second, we develop an efficient algorithm that achieves the optimal regret order offered by centralized learning using the minimum order of communications dictated by the information-theoretic lower bound. For sparse linear bandits, we show a variant of the proposed algorithm offers better regret-communication trade-off by leveraging the sparsity of the problem. \\

For the distributed linear bandit problem, to achieve the optimal regret order of $\Omega(d\sqrt{MT})$ in both $M$ and $T$ as offered by centralized learning, the agents need to cooperate in learning the underlying reward vector $\theta^*$. In addition to a policy for choosing reward-generating actions at each time, a distributed learning algorithm also includes a communication strategy that governs \emph{when} to communicate and \emph{what} and \emph{how} to send it (i.e., quantization and encoding) over the finite-capacity channels. To minimize the total regret that is accumulating over time and aggregating over the agents while using a minimum amount of communications, the communication strategy needs to work in tandem with action selection to ensure a continual flow of information available at all agent for decision-making.  \\

The key idea of the proposed algorithm is an \emph{progressive learning and sharing} (PLS) structure that systematically coordinates the collective exploration of $\theta^*$ and the information sharing of the estimates of $\theta^*$ over finite-capacity channels. Specifically, the PLS algorithm progresses as each agent learns and shares one-bit information about $\theta^*$ (per dimension) at a time, starting from the most significant bit in the binary expansion of $\theta^*$ in each dimension. This bit-by-bit learning and sharing is structured in interleaving exploration and exploitation epochs with carefully controlled epoch lengths, to achieve both the minimum order of channel usage and the minimum order of cumulative regret. 

\subsection{Related Work}

Communication cost is commonly partitioned into two parts: the size of the message at each information exchange and the frequency of information exchange. As detailed below, these two sides of the same coin have largely been dealt with separately in the literature when developing communication-efficient algorithms for distributed bandit problems. Our work departs from existing studies by taking a holistic view on communication cost and making an initial attempt at characterizing the information-theoretic limit on the communication-learning trade-off in distributed linear bandits. \\

In the group of work focusing on reducing communication frequency through intermittent information exchange, it is often assumed that the information being transmitted, typically a vector in $\R^d$, can be communicated with infinite precision, which requires a channel with infinite capacity. See, for example, \citep{Hillel2013, Tao2019, Agarwal2021} on discrete bandit problems and \cite{Wang2019, Ghosh2021, Huang2021, Chawla2022lowdimensional, Amani2022} on linear bandits. In particular, \cite{Wang2019, Huang2021, Amani2022} proposed algorithms based on batched elimination of arms where poorly performing arms are eliminated at the end of each batch. The batched structure offers a natural way to limit communication to one message exchange per batch. However, there is no constraint on the amount of information that can be transmitted in each batch. The algorithms in the aforementioned papers also require solving a G-optimal design problem before each communication round, which can be computationally expensive as the dimension increases. In contrast, the PLS algorithm developed in this work has low computational complexity that scales linearly with the dimension $d$. \\

The other group of work focuses on reducing the size of the message at each information exchange. The frequency of communication is not an concern. The objective is to best \emph{approximate} the information being exchanged via techniques such as quantization and sparsification (see, for example, \cite{Konecny2016quantization, Hanna2021, Mitra2022bitconstrained, Suresh2017}). In particular, \citep{Hanna2021} proposed a quantization scheme to reduce the communication overhead in discrete multi-armed bandit problems at the cost of a small multiplicative constant in the regret. Recently,~\citep{Mitra2022bitconstrained} proposed an algorithm for decentralized linear bandits with a finite-capacity uplink channel and an infinite-capacity downlink channel. They developed an adaptive encoding scheme for communication that ensured order-optimal regret for their proposed algorithm. However, with a linear order of message exchanges in $T$, the total uplink communication cost is $\cO(dT)$ bits as opposed to the $\cO(d \log T)$ bits of communication cost in our proposed algorithm. Moreover, the results in~\cite{Mitra2022bitconstrained} hold only for single agent while ours hold for a distributed setup with multiple agents.  \\

In the context of developing communication-efficient algorithm, another line of related work is Federated Learning (FL)~\citep{mcmahan2017communication}. FL aims to collaboratively learn a model by leveraging the data available at all the agents with a focus on ensuring privacy of the data for the participating agents. Developing communication efficient FL algorithms is an active area of research (see \cite{Konecny2016quantization, Liu2019FedBCD, Sun2019communication, Reisizadeh2020fedpaq, Haddadpour2021federated, Jhunjhunwala2021adaptive, Honig2022DAdaQuant} and references therein). Detailed surveys can be found in~\citep{Tang2020communication} and~\citep{Zhao2022communication}. These studies focus on the first-order stochastic optimization, which is different from the (zeroth-order) stochastic linear bandits considered in this work, in terms of both action selection strategies and the relevant information that needs to be exchanged. \\

It is impossible to do full justice to the vast literature on communication-efficient distributed learning. We present above existing studies on distributed bandits that are most relevant to this work. Additional discussions of related work is provided in Appendix~\ref{sec:related_work_appendix}, albeit remaining to be incomplete. 

\section{Problem Formulation}
\label{sec:problem_formulation}

Consider a system of $M$ distributed agents indexed by $\{1,2,\dots, M\}$. The agents face a common stochastic linear bandit model characterized by an unknown mean reward vector $\theta^* \in \R^d$. Specifically, each agent $j \in \{1,2, \dots, M\}$ has access to an action set $\cA = \{a \in \R^d: \|a\|_2 \leq 1\}$ and chooses to play an action $a_t^j \in \cA$ at every time instant $t$ during a time horizon of $T$ instants. When an action $a_t^j \in \cA$ is played by agent $j$ at time $t$, it receives a reward 
\[
y_t^j = \ip{\theta^*}{a_t^j} + \eta_t^j ,
\]
where $\eta_t^j$ is zero-mean noise that is i.i.d. across time instants and across the agents and satisfies $\log(\E[\exp(\lambda \eta_t^j)]) \leq \lambda^2 \sigma^2/2$ for all $\lambda \in \R$, i.e., the noise is $\sigma^2$-sub-Gaussian. WLOG, we assume $\| \theta^* \|_2 \leq 1$. It is straightforward to extend it to the case $\| \theta^* \|_2 \leq B$, where $B$ is a known constant, by appropriately scaling the obtained rewards. \\

Information exchange across the agents goes through a central server. Both the uplink channel (from the agents to the server) and downlink channel (from the server to the agents) have a finite capacity of $R$ bits per channel use, which limits the message size in each information exchange. This model quantifies the cumulative communication cost to the bit level, and represents a more challenging problem than those considered in the literature where communication channel between the server and the agent is assumed to have infinite capacity in at least one direction, if not both. \\

The design objective is a distributed learning policy consisting of (i) a decision strategy that governs the selection of actions $\{a_t^j\}$ of each agent $j$ at each time $t$ and (ii) a communication strategy that determines when to communicate what and how to send it over the channel via quantization and encoding. The performance of a learning policy is measured in terms of the overall cumulative regret $R(T)$  and the cumulative communication cost $C(T)$ incurred by the policy. The overall cumulative regret is given by
\begin{align}
    R(T) = \sum_{j = 1}^M \sum_{t = 1}^T \left[  \max_{a \in \cA} \ip{\theta^*}{a} - \ip{\theta^*}{a_t^j} \right].
\end{align}
The communication cost $C(T)$ is measured using $C_{\text{u}}(T)$ and $C_{\text{d}}(T)$, the number of bits transmitted on the uplink channel (i.e., by any agent to the server) and that on the downlink channel (i.e., the average number of bits transmitted by the server to an agent), respectively. \\

The learning and communication efficiency of a learning policy is measured against the benchmarks. In particular, the cumulative regret is lower bounded by
$\Omega(d\sqrt{MT})$, which is the optimal regret order in a centralized setting with total $MT$ reward observations centrally available for learning. In Sec.~\ref{sec:lower_bound}, we establish an information-theoretic lower bound on the communication cost required for achieving a sublinear regret order. 

\section{Progressive Learning and Sharing}
\label{sec:algorithm_description}

In this section, we present the Progressive Learning and Sharing (PLS) algorithm. We start in Sec.~\ref{sub:basic_structure_PLS} with the basic structure of the algorithm followed by a detailed implementation in Sec.~\ref{sub:PLS_detailed}. 

\subsection{The Basic Structure of PLS}
\label{sub:basic_structure_PLS}

In PLS, learning and information sharing progress along the binary expansion of $\theta^*$ in each dimension, starting from the most significant bit. Below we present separately the information sharing and learning components of PLS.

\subsubsection{Progressive Information Sharing}
\label{subs:PRINS}

In PLS, the unknown reward vector $\theta^*$ is learnt with increasing accuracy, one bit at a time, as the algorithm progresses. Once the next bit in the binary expansion\footnote{While the algorithm learns an one bit at a time, it does not necessarily imply that the bit sequence learnt corresponds to the binary expansion.} of $\theta^*$ in each of the $d$ coordinates is learnt with sufficient accuracy, the agents transmit their estimates of this bit to the central server, which aggregates the estimates and broadcast the aggregated estimate of the new bit to the agents for subsequent exploitation and further exploration of $\theta^*$. \\

This progressive sharing mechanism can be seamlessly integrated with regret minimization to achieve both minimum regret order and minimum channel usage. Specifically, since only $1$ bit of information is shared per coordinate in each information exchange, it suffices to send $d$ bits in each transmission, achieving the benefit of having small messages. Furthermore, one can note that any reward-maximizing action taken based on an estimate $\hat{\theta}$ incurs a regret proportional to the estimation error $\|\hat{\theta} - \theta^*\|_2$. Consequently, an estimation error of $\cO(1/\sqrt{T})$ is sufficient to ensure an order-optimal regret, implying that it is sufficient to estimate each coordinate of $\theta^*$ up to an accuracy of $\cO(1/\sqrt{T})$. Since sending the first $r$ bits of the binary representation ensures an error of no more than $2^{-r}$, transmitting the first $\cO(\log T)$ bits of the binary representation is sufficient to achieve the required accuracy. As a result, infrequent communication with a total of $\cO(\log T)$ rounds can transmit all relevant information about $\theta^*$. 

\subsubsection{Progressive Collaborative Learning}
\label{subs:ProCoL}

The progressive learning component of PLS is carried out in two stages: an initial stage for estimating the norm of $\theta^*$ followed by a refinement stage with interleaving exploration and exploitation. \\

In the initial norm estimation stage, the goal is to estimate, within a multiplicative factor, the norm $\|\theta^*\|_2$ of the underlying mean reward. This procedure is purely exploratory in nature. The collaborative exploration across agents is carried out in epochs with exponentially growing epoch lengths. At the end of each epoch, a threshold-based termination test is employed to determine whether the required estimation accuracy has been reached, which terminates the norm estimation stage. Information exchange occurs at the end of each epoch, and the exponentially growing epoch length ensures a low communication frequency. \\

The norm estimation stage serves multiple purposes. First, it allows PLS to be adaptive  to the norm of $\theta^*$ through the threshold-based termination rule. Specifically, it provides sufficient initial exploration with sufficiency automatically adapted to $\|\theta^*\|_2$ to ensure the estimation error is small enough for subsequent exploitation in the refinement stage. Second, this initial norm estimate sets the dynamic range of subsequent estimates of $\theta^*$ to be used in the differential quantization for subsequent information sharing. Third, the estimate of $\|\theta^*\|_2$ is also used to control the length of the exploitation epoch in the refinement stage to balance the exploration-exploitation trade-off. \\

In the refinement stage, the estimate of $\theta^*$ obtained in the norm estimation stage is further refined. Similar to the norm estimation stage, the refinement stage also proceeds in epochs. The difference is that each epoch in the refinement stage consists of an exploration sub-epoch followed by an exploitation one. The exploration sub-epochs are for continual learning of $\theta^*$, one bit in each sub-epoch. The exploitation sub-epochs are to maximize rewards at each agent by playing the best action based on the current estimate of $\theta^*$. The lengths of the exploration and exploitation sub-epochs are both growing exponentially, but at different rates to carefully balance the exploration-exploitation trade-off. Information sharing is carried out only at the end of each exploration sub-epoch for the newly learned bit of $\theta^*$. The refinement stage with its interleaved exploration and exploitation continues until the end of the time horizon.

\subsection{Detailed Description of PLS}
\label{sub:PLS_detailed}

In this section, we dive into the details of PLS.

\subsubsection{Progressive Collaborative Learning}

\paragraph{Norm Estimation Stage:} This stage proceeds in purely exploratory epochs. During an epoch $k$, each agent plays each unit vector in an orthonormal basis\footnote{The basis is chosen \emph{a priori} and known to all the agents and the server. It can be any orthonormal basis of $\R^d$.} of $\R^d$ for $s_k$ times. Each agent $j$ computes the sample mean of the observed rewards for each basis vector to obtain an estimate $\hat{\theta}_{k}^{(j)}$ of the underlying vector $\theta^*$. This estimate $\hat{\theta}_{k}^{(j)}$ is clipped to a length of $R_k + B_k$, quantized using a stochastic quantizer with resolution $\alpha_k$ and sent to the server by the agent. The process is repeated in every epoch until the agents receive a message from the server to terminate. We defer the details of the clipping and quantization steps to the next section that describes the communication strategy. All policy parameters are specified at the end of the section. \\

\begin{algorithm}[H]
    \caption{Norm Estimation:  Agent $j \in \{1,2,\dots, M\}$}
    \label{alg:norm_est_agent}
    \begin{algorithmic}[1]
        \STATE Set $k \leftarrow 1$
        \WHILE{\texttt{True}}
            \STATE Play each basis vector $s_k$ times and compute the sample mean $\hat{\theta}^{(j)}_k$
            \STATE $\tilde{\theta}^{(j)}_k \leftarrow \textsc{Clip}(\hat{\theta}^{(j)}_k, R_k + B_k)$
            \STATE $Q(\tilde{\theta}^{(j)}_k) \leftarrow \textsc{StoQuant}(\tilde{\theta}^{(j)}_k, \alpha_k, R_k + B_k)$
            \STATE Send $Q(\tilde{\theta}^{(j)}_k)$ to the server
            \IF{received \texttt{terminate} from server}
            \STATE \textbf{break}
            \ELSE
            \STATE $k \leftarrow k + 1$
            \ENDIF
        \ENDWHILE
    \end{algorithmic}
\end{algorithm}

At the server, upon receiving the estimates from the agents, the server averages them to obtain a combined estimate $\hat{\theta}_{k}^{(\textsc{serv})}$. The server compares the norm of this estimate to a threshold $4\tau_k$. If the norm exceeds the threshold, the server sends a message to the agents to terminate the norm estimation stage. Otherwise, the server and the agents proceed into the next epoch. The value of $\tau_k$ is chosen to be an upper bound on the estimation error of $\theta^*$ at the end of the $k^{\text{th}}$ epoch, allowing PLS to estimate $\|\theta^*\|_2$ within a multiplicative factor at the end of the norm estimation stage. \\

The pseudo code for the norm estimation stage is given in Algorithms~\ref{alg:norm_est_agent} and ~\ref{alg:norm_est_server}.

\begin{algorithm}[H]
	\caption{Norm Estimation: The Server}
	\label{alg:norm_est_server}
	\begin{algorithmic}[1]
        \STATE Set $k \leftarrow 1$
        \WHILE{\texttt{True}}
            \STATE Compute $\hat{\theta}_{k}^{(\textsc{serv})} = \frac{1}{M} \sum_{j = 1}^M Q(\tilde{\theta}^{(j)}_k)$
            \IF{$\tau_k \leq \frac{1}{4}\|\hat{\theta}_{k}^{(\textsc{serv})}\| $}
            \STATE Server sends terminate to all agents
            \STATE \textbf{break}
            \ELSE
            \STATE $k \leftarrow k + 1$
            \ENDIF
        \ENDWHILE
	\end{algorithmic}
\end{algorithm}

\begin{algorithm}
    \caption{Refinement:  Agent $j \in \{1,2,\dots, M\}$}
    \label{alg:ref_est_agent}
    \begin{algorithmic}[1]
            \STATE \textbf{Input}: The epoch index at the end of Norm Estimation stored as $k_0$, $\bar{\theta}_{k_0 - 1} \leftarrow 0, k \leftarrow k_0$
            \WHILE{budget is not exhausted}
                \STATE Play each basis vector $s_k$ times and compute the sample mean $\hat{\theta}^{(j)}_k$
                \STATE $\tilde{\theta}^{(j)}_k \leftarrow \textsc{Clip}(\hat{\theta}^{(j)}_k - \bar{\theta}_{k - 1}, R_k + B_k)$
                \STATE $Q(\tilde{\theta}^{(j)}_k) \leftarrow \textsc{StoQuant}(\tilde{\theta}^{(j)}_k, \alpha_k, R_k + B_k)$
                \STATE Send $Q(\tilde{\theta}^{(j)}_k)$ to the server
                \STATE Receive $Q(\hat{\theta}_{k}^{(\textsc{serv})})$ from the server
                \STATE $\bar{\theta}_{k} \leftarrow \bar{\theta}_{k - 1}  + Q(\hat{\theta}_{k}^{(\textsc{serv})})$
                \IF{$k = k_0$}
                \STATE Set $\mu_0 \leftarrow \|\bar{\theta}_{k}\|_2$
                \ENDIF
                \STATE Play the action $a = \bar{\theta}_{k}/\|\bar{\theta}_{k}\|$ for the next $t_{k}$ rounds.
                \STATE $k \leftarrow k + 1$
            \ENDWHILE
    \end{algorithmic}
\end{algorithm}

\paragraph{Refinement Stage:} This stage also proceeds in epochs, starting with the epoch index at which Norm Estimation stage terminated. Each epoch $k$ during Refinement begins with an exploration sub-epoch where, similar to Norm Estimation, each agent obtains their estimate of $\theta^*$, $\hat{\theta}_{k}^{(j)}$, by playing each of the basis vectors $s_k$ times. At the end of the sub-epoch, the agents share the next bit learnt during this time by transmitting $\hat{\theta}_{k}^{(j)}- \bar{\theta}_{k-1}$ to the server after appropriate clipping and quantization. Here $\bar{\theta}_{k-1}$ denotes the current estimate of $\theta^*$ available to the agents after $k - 1$ epochs. This differential quantization allows the agents to share only the ``new bit" learnt during the exploration sub-epoch. As a response, the agents receive $Q(\hat{\theta}_{k}^{(\textsc{serv})})$, a quantized version of the udpate, from the server which is used to refine their estimate of $\theta^*$ to $\bar{\theta}_k = \bar{\theta}_{k-1} + Q(\hat{\theta}_{k}^{(\textsc{serv})})$. The exploitation sub-epoch follows the communication round where all agents play the unit vector along $\bar{\theta}_k$ throughout the $t_k$ time steps of the sub-epoch. The refinement stage continues by proceeding into the next epoch and repeating the process until the end of the time horizon. \\

The steps at the server in this stage are similar to those in the norm estimation stage. In particular, the server collects estimates from the agents at the end of the exploration sub-epoch, computes the mean $\hat{\theta}_{k}^{(\textsc{serv})}$ and then broadcasts the differential update $\hat{\theta}_{k}^{(\textsc{serv})} - \bar{\theta}_{k - 1}$ after passing it through a deterministic quantizer with resolution $\beta_k$. \\

A pseudo code for the refinement stage is given in Algorithms~\ref{alg:ref_est_agent} and~\ref{alg:ref_est_server}.

\begin{algorithm}
    \caption{Refinement: The Server}
    \label{alg:ref_est_server}
    \begin{algorithmic}[1]
            \STATE \textbf{Input}: The epoch index at the end of Norm Estimation stored as $k_0$, $\bar{\theta}_{k_0 - 1} \leftarrow 0, k \leftarrow k_0$
            \WHILE{time horizon $T$ is not reached}
                \STATE Receive $Q(\tilde{\theta}^{(j)}_k)$ from all the agents
                \STATE Compute $\hat{\theta}_{k}^{(\textsc{serv})} = \bar{\theta}_{k - 1} + \frac{1}{M} \sum_{j = 1}^M Q(\tilde{\theta}^{(j)}_k)$
                \STATE $Q(\hat{\theta}_{k}^{(\textsc{serv})}) \leftarrow \textsc{DetQuant}(\hat{\theta}_{k}^{(\textsc{serv})} - \bar{\theta}_{k - 1}, \beta_k, B_k + \tau_{k})$ and broadcasts it to all agents
                \STATE $k \leftarrow k + 1$
            \ENDWHILE
    \end{algorithmic}
\end{algorithm}

\paragraph{Setting Policy Parameters:}
We now specify the values of parameters used in PLS. For an epoch $k$, the length of the exploration (sub-)epoch, $s_k$, is set to $\lceil 40 \sigma^2 d \log(8MK/\delta) 4^k \rceil$ and that of the exploitation one is set to $t_k := \lceil Ms_k^2 \mu_0^2 \rceil$. In the above definitions, $K$ denotes the maximum possible number of epochs in the algorithm and is defined as $K := \max \{ k \in \N : 40 \sigma^2 d \log(8Mk/\delta) (4^k - 4) \leq T \} = \cO(\log T)$. In the definition of $t_k$, $\mu_0$ is the estimate of $\|\theta^*\|_2$ obtained at the end of the norm estimation stage. The exponential lengths of the epoch designed for the bit by bit progressive learning are evident from the above choices. This choice of the lengths also allows PLS to address the exploration-exploitation trade-off by balancing the regret incurred during the exploration and exploitation sub-epochs.  \\

The threshold $\tau_k$ and sequence $R_k$ are set based upon high probability bounds on $\|\hat{\theta}_{k}^{(\textsc{serv})} - \theta^*\|_2$ and $\|\hat{\theta}_k^{(j)} - \theta^*\|_2$ respectively that simultaneously hold for all agents. In particular, $\tau_k$ is set to $3 \cdot 2^{-(k+1)}/\sqrt{M}$ and $R_k := 2^{-k}$. Notice that this choice of $R_k$ echoes the progressive learning feature, allowing the agents to learn $\theta^*$, one bit at a time. The sequence $B_k$ bounds the error $\|\bar{\theta}_{k-1} -\theta^*\|_2$ and is set to $5\tau_k$ for $k \geq 2$ with $B_1 = 1$. Lastly, the resolution parameter sequences $\alpha_k$ and $\beta_k$ are defined as $\alpha_k := \alpha_0 \sigma \sqrt{d}/\sqrt{s_k}$ and $\beta_k = \beta_0 \tau_k$ for some numerical constants $\alpha_0, \beta_0 < 1$. The constants $\alpha_0$ and $\beta_0$ control the message size associated with uplink and downlink communication respectively.

\subsubsection{Progressive Information Sharing}
\label{subs:comm_scheme_details}

The communication protocol of PLS consists of two steps : clipping and quantizing the vector to be sent to reduce the size of message being transmitted and encoding the quantized value to send it over the communication channel. \\

\paragraph{Clipping and Quantization}: This step employs well-known sub-routines described below to map a vector to a low resolution, quantized version of itself.

\begin{itemize}
    \item $\textsc{Clip}(x,r)$ is a simple routine that takes input a vector $x$ and clips it within a $\ell_2$ ball of radius $r$. Mathematically, the routine returns the value $x \cdot \min\{1, r/\|x\|\}$.
    \item $\textsc{StoQuant}(y, \varepsilon, r)$ returns the quantized version of a scalar $y$ using the popular approach of stochastic quantization. Specifically, the interval $[-r,r]$ is divided into $l_{\varepsilon} = \lceil2r/\varepsilon\rceil$ intervals of equal length and indexed from $1$ to $l_{\varepsilon}$. The value $y$ is quantized to one of the end points of the intervals to which it belongs in a randomized manner with the probability inversely proportional to the distance from $y$. In particular, the stochastic quantizer outputs $Q_s(y)$ given by 
    \begin{align*}
        Q_s(y) = \begin{cases} b_{l-1} & \text{w.p. } b_l - y, \\ b_l & \text{otherwise.}\end{cases}
    \end{align*}
    In the above expression, $b_m = r \left( \dfrac{2m}{l_{\varepsilon}} - 1 \right)$ for $m = 1,2, \dots, l_{\varepsilon}$ and $l = \{m: b_{m-1} \leq y < b_m\}$. With a slight abuse of notation, we also use $\textsc{StoQuant}(x, \varepsilon', r)$ to denote stochastic quantization of a vector $x$. In the case of a vector, all the coordinates are quantized as mentioned above with $\varepsilon = \varepsilon'/\sqrt{d}$.
    \item $\textsc{DetQuant}(y, \varepsilon, r)$ is also a quantization routine similar to $\textsc{StoQuant}(y, \varepsilon, r)$ with the only difference that the output of this routine is deterministic. Specifically, the deterministic quantizer outputs
    \begin{align*}
        Q_d(y) = \begin{cases} b_{l-1} & \text{if } |b_l - y| > |b_{l-1} - y|, \\ b_l & \text{otherwise.}\end{cases}
    \end{align*}
    Once again we overload the definition with a vector analogue $\textsc{DetQuant}(x, \varepsilon', r)$ where each coordinate is deterministically quantized with $\varepsilon = \varepsilon'/\sqrt{d}$.
\end{itemize}

The two different quantization schemes used in this step serve their own, different purposes in algorithm. PLS employs stochastic quantization for uplink communication which allows it to exploit the concentration properties of the zero mean noise added by the quantization. Consequently, it enables to completely leverage the access to observations from $M$ different agents to achieve the speed up proportional to the number of agents. A deterministic quantization scheme in its place would have resulted in accumulation of errors and consequently prevented the speedup with the number of agents. On the other hand, the deterministic quantization used for downlink transmission ensures that all the agents have the same estimate of $\theta^*$ and consequently the same value of $\mu_0$, the estimate of $\|\theta^*\|_2$ . Since $\mu_0$ governs the length of the exploitation sub-epoch, a common value shared between the agents helps maintain the synchronization across different epochs and agents. \\

\paragraph{Encoding}: As described above, both the quantization routines used in PLS quantize each coordinate axis into $l_{\varepsilon}$ intervals implying each coordinate of the quantized version takes one of the possible $l_{\varepsilon} + 1$ values, for any given value of accuracy parameter $\varepsilon$. The encoding step maps these $(l_{\varepsilon} + 1)^d$ possible values of quantized vectors to different messages that can be sent over the communication channel. We use a common encoding strategy for the uplink and downlink channels. \\

PLS sends each coordinate of the quantized version, one by one, using the variable-length encoding strategy, unary coding. In particular, each coordinate is represented by a header followed by a sequence of $1$'s whose length is equal to the absolute value of the coordinate. The header is $3$ bits long where the first and the third bit are both $0$ and the second bit represents the sign of the value, where $0$ implies a negative value while $1$ implies a positive one. As an example, in this encoding scheme, the numbers $-3$ and $5$ are represented as $000111$ and $01011111$ respectively.

\section{Performance Analysis}
\label{sec:analysis}

In this section, we show that PLS achieves the order-optimal regret of $\cO(d\sqrt{MT})$ up to logarthmic factors with a communication cost that matches the order of information-theoretic lower bound established in Sec.~\ref{sec:lower_bound}.

\subsection{Regret Analysis}
\label{sub:regret_analysis}

The following theorem characterizes the regret performance of PLS.

\begin{theorem}
Consider the distributed stochastic linear bandit setting described in Sec.~\ref{sec:problem_formulation}. If PLS is run with parameters as described in Sec.~\ref{sub:PLS_detailed} with a budget of $T$ queries, then the following relation holds with probability at least $1 - \delta$,
\begin{align*}
    R_{\text{PLS}}(T) & \leq C d\sqrt{MT} \log \left(\frac{8MK}{\delta} \right) \log\left({9 \sqrt{MT}}\right) \\
     & ~~~~~~~~~~~~~~~~~~ + \cO(\log T),
\end{align*}
for some constant $C > 0$, independent of $d, M$ and $T$.
\label{theorem:regret_bound}
\end{theorem}

Theorem~\ref{theorem:regret_bound} establishes that the regret incurred by PLS is $\tilde{\cO}(d \sqrt{MT})$ which matches the lower bound for a centralized setting with $MT$ queries up to logarithmic factors. This implies that PLS explores the achievability of communication-learning trade-off at the frontier of optimal learning performance. \\

We here provide a sketch of the proof for Theorem~\ref{theorem:regret_bound}. We begin the proof by decomposing the regret incurred by PLS as follows : $R_{\text{PLS}}(T) = R_{\textsc{NE}}(T) + R_{\textsc{REF}}(T)$, where $R_{\textsc{NE}}(T)$ and $R_{\textsc{REF}}(T)$ denote the regret incurred during the norm estimation and the refinement stages respectively. The bound on regret incurred by PLS is obtained by separately bounding each of the two terms in the decomposition, $R_{\textsc{NE}}(T)$ and $R_{\textsc{REF}}(T)$. The following lemma provides a bound on the regret incurred during the norm estimation stage.
\begin{lemma}
Consider the Norm Estimation stage described in Alg.~\ref{alg:norm_est_agent} and~\ref{alg:norm_est_server}. If it is run for at most $T$ time instants in a distributed setup with $M$ agents where the underlying mean reward satisfies $\|\theta^*\| \leq 1$, then the regret incurred during this stage satisfies the following relation with probability at least $1 - \delta$,
\begin{align*}
    R_{\textsc{NE}}(T) & \leq C d(1 + \sigma^2) \sqrt{MT} \log(1/\delta') \log_2 \left(9\sqrt{MT} \right)  \\
    &  ~~~~~~~~~~~~~~~~ +2d  \log_2 (9\sqrt{MT}),
\end{align*}
where $\delta' = \delta/8MK$ and $C > 0$ is a constant independent of $d, M$ and $T$.
\label{lemma:regret_norm_est}
\end{lemma}

Since the Norm Estimation stage is based on pure exploration, we upper bound $R_{\textsc{NE}}(T)$ by $2\|\theta^*\|_2$ times the duration of the norm estimation stage, where $2\|\theta^*\|_2$ corresponds to the trivial bound on the instantaneous regret. The central step in the proof of the above lemma is bounding the duration of the norm estimation stage. For this part, we first establish a $\cO(1/\|\theta^*\|_2^2)$ bound on the duration based on our threshold test which determines when the stage terminates. However, the fixed length of the time horizon dictates a hard upper bound of $T$  on the duration of this stage. The proof is completed by taking a minimum of these two bounds to bound the duration of the norm estimation stage followed optimizing over the choice of $\|\theta^*\|_2$ to obtain the tightest bound. \\

To bound $R_{\textsc{REF}}(T)$, we separately bound the regret incurred during the exploration and exploitation sub-epochs. The regret incurred during the exploration sub-epoch of the $k^{\text{th}}$ epoch is bounded by $2\|\theta^*\|_2$ times $ds_k$, the duration of the exploration sub-epoch, similar to the proof of Lemma~\ref{lemma:regret_norm_est}. The regret incurred during the exploitation sub-epoch is bounded with the help of the following lemma.
\begin{lemma}
If $\hat{\theta}$ is an estimate of a vector $\theta$ such that $\|\hat{\theta} - \theta\|_2 \leq \tau \leq \|\theta\|_2$, then instantaneous regret incurred by an algorithm that plays the action $a = \hat{\theta}/\|\hat{\theta}\|_2$ on a stochastic linear bandit instance with true underlying vector $\theta$ can be bounded by $\tau^2/\|\theta\|$.    
\label{lemma:theta_error_regret}
\end{lemma}
Lemma~\ref{lemma:theta_error_regret} implies that even when the estimation error is $\nu \|\theta^*\|_2$, for $\nu \in [0,1]$ the regret incurred by the algorithm scales as $\nu^2 \|\theta^*\|_2$. This is a crucial fact that helps balance the exploration-exploitation trade-off. In particular, the lengths of the exploration and exploitation epochs in PLS are designed based on the result of this lemma. The estimation error at the end of exploration sub-epoch in epoch $k$ satisfies $\cO(s_k^{-1/2})$ while the regret incurred during the sub-epoch satisfies $\cO(s_k \|\theta^*\|_2)$. Based on the above lemma, the regret incurred during the corresponding exploitation sub-epoch of length $t_k = \cO(s_k^2 \|\theta^*\|_2^2)$ is $\cO(\frac{1}{s_k \|\theta^*\|_2} \cdot t_k) = \cO(s_k \|\theta^*\|_2)$ matching the regret incurred during the exploration phase. This is the fundamental mechanism that provides the necessary balance between exploration and exploitation in PLS allowing it to achieve optimal-order regret. Moreover, this also provides additional insight into the novel choice of the length exploitation epoch based on the norm of $\theta^*$ in PLS. \\

The final bound on $R_{\textsc{REF}}(T)$ is obtained by noting each epoch is $\cO(s_k^2 + s_k) = \cO(16^k)$ steps long implying a total of $\cO(\log T)$ epochs. The detailed proofs of all the Lemmas and the Theorem can be found in Appendix~\ref{sec:PLS_analysis}.

\subsection{Communication Cost}

The communication cost incurred by PLS is characterized in the following theorem. 

\begin{theorem}
Consider the distributed stochastic linear bandit setting described in Sec.~\ref{sec:problem_formulation}. If PLS is run with parameters as described in Sec.~\ref{sec:algorithm_description} for a time horizon of $T$, then the uplink and the downlink communication costs (in bits) incurred by PLS, i.e. $C_{\text{u}}(T)$ and $C_{\text{d}}(T)$, satisfy $\cO\left( \dfrac{d}{\alpha_0} \log T\right)$ and $\cO\left( \dfrac{d}{\beta_0} (\log M + \log T)\right)$ respectively.
\label{theorem:comm_bound}
\end{theorem}

Thus, Theorem~\ref{theorem:comm_bound} in conjunction with Theorem~\ref{theorem:regret_bound} shows that PLS incurs the order-optimal regret while simultaneously achieving the order-optimal communication cost matching the lower bound in Sec.~\ref{sec:lower_bound}. Hence, PLS further reduces the communication cost as compared to communication-efficient algorithms proposed in~\cite{Wang2019, Huang2021, Amani2022} while maintaining the optimal regret performance. The proof of the above theorem revolves around the following lemma.
\begin{lemma}
Consider the communication scheme of PLS outlined in Sec.~\ref{subs:comm_scheme_details}. If PLS is run with parameters as described in Sec.~\ref{sub:PLS_detailed}, then any message exchanged between the server and an agent during PLS is at most $\cO(d)$ bits.
\label{lemma:channel_capacity}
\end{lemma}
The bound on uplink communication cost $C_{\text{u}}(T)$ immediately follows by noting that there are at most $K = \cO(\log T)$ epoch, or equivalently, communication rounds in PLS. The additional $\cO(d \log M)$ term in $C_{\text{d}}(T)$ is incurred when the server sends the initial estimate of $\theta^*$ to all the agents. The details of the proof along with proof of Lemma~\ref{lemma:channel_capacity} are provided in Appendix~\ref{sec:PLS_analysis}.

\begin{remark}
     Based on Lemma~\ref{lemma:channel_capacity}, it can immediately concluded that a capacity of $R = \cO(d)$ bits for both the uplink and the downlink channel suffices for PLS to achieve order-optimal regret performance. Some other works like~\citep{Suresh2017} and~\cite{Mitra2022bitconstrained} have also proposed encoding schemes which result in message sizes of $\cO(d)$ bits.~\citep{Suresh2017} proposed an encoding scheme for distributed mean estimation using the well-known variable length encoding schemes such as Huffman and Arithmetic encoding. It first constructs a histogram over the different $l_{\varepsilon} + 1$ values in the qunatized version of a vector followed by a Huffman tree based on that histogram. During each communication round, the sender first sends the corresponding Huffman tree followed by the message encoded using that. Compared to such variable-length schemes, our proposed scheme is easier to implement, both in terms of memory and computation, and also avoids overheads like sending the Huffman tree. The IC-Lin-UCB algorithm proposed in~\cite{Mitra2022bitconstrained} uses an encoding scheme based on constructing $\varepsilon$-cover of hyperspheres in $\R^d$ at every time instant. The computational cost of constructing such covering sets grows exponentially with the dimension, rendering the approach infeasible even for problems of moderate dimensionality. On the other hand, the computational cost of the encoding scheme in PLS grows linearly with dimension, making PLS an attractive option even for high dimensional problems.
\end{remark}

\begin{remark}
    The $\cO(d \log M)$ term in the downlink communication cost can be interpreted as the cost incurred by the server to facilitate information exchange since the data is distributed across $M$ agents. In other words, the server provides each agent with the additional information learnt from the other agents through these $\cO(d \log M)$ bits, leading to \emph{collaboration} among them. In absence of transmission of these bits, the problem would reduce to $M$ independent agents trying to learn a linear bandit model and incurring an overall regret of $\cO(dM\sqrt{T})$ that grows linearly with $M$. Thus, it can also be interpreted as the cost associated with having a sublinear regret with respect to the number of agents $M$. Similarly, $\cO(d \log T)$ term corresponds to the cost associated with having a sublinear regret with respect to the time horizon.    
\end{remark}

\section{Lower Bound on Communication Cost}
\label{sec:lower_bound}

In this section, we explore the converse result for the communication-learning trade-off. In particular, the following theorem establishes information theoretic lower bounds in terms of actual number of bits that need to be transmitted by the clients and the server over the channel for any distributed algorithm to achieve sublinear regret.

\begin{theorem}
    Consider the distributed linear bandit instance with $M$ agents described in Section~\ref{sec:problem_formulation}. Any distributed algorithm that incurs an overall cumulative regret that in sublinear in both $T$ and $M$ with probability at least $2/3$ needs to transmit at least $\Omega(d \log(MT))$ bits of information over the downlink channel and $\Omega(d \log T)$ bits of information over the uplink channel.
    \label{theorem:comm_lower_bound}
\end{theorem}

The lower bounds established in the above theorem match the achievability results for PLS shown in the previous section. We would like to emphasize that PLS is the first algorithm for distributed linear bandits for which communication cost incurred matches the order of information theoretic lower bounds in terms of actual number of bits transmitted over the channel. Additionally, PLS simultaneously attains order-optimal regret guarantees. Thereby, the above theorem along with the performance guarantees of PLS provides a holistic view of the communication-learning efficiency trade-off in distributed linear bandits. The proof of the theorem follows from an application of Fano's inequality along with bounds on metric entropy. A detailed proof of the above theorem is provided in Appendix~\ref{sec:lower_bound_proofs}.
\section{Leveraging Sparsity}

We now consider a variant of the original problem where the underlying reward vector $\theta^*$ is known to satisfy an additional sparsity constraint. In particular, it is known that the number of non-zero elements in $\theta^*$ are no more than $s \ll d$, i.e., $\|\theta^*\|_0 \leq s$. For this setup, we propose Sparse-PLS, a variant of PLS, to leverage the sparsity of $\theta^*$ to further reduce the communication cost.  \\

Sparse-PLS makes two modifications to the original PLS algorithm to leverage the sparsity of $\theta^*$. The first modification is made to the set of actions played during an exploration epoch. Specifically, Sparse-PLS replaces the orthonormal basis of $R^d$ with a set of actions, $\cB_s$, that spans only a subspace of $R^d$. $\cB_s$ consists of $m =\cO(s \log(d/\delta)) \ll d$ vectors drawn independently from the set $\{-1/\sqrt{d}, +1/\sqrt{d}\}^d$. It is ensured that all the agents use the same random set $\cB_s$ via a common random seed. This modification offers a two-fold advantage. First, it helps reduce the message size required for uplink communication as it is sufficient to transmit these noisy projections in $\R^m$, where $m \ll d$, in order to recover the original sparse vector $\theta^*$~\cite{Candes2006, Candes2006IEEE, Candes2007}. Second, since the regret incurred in PLS is proportional to length of the exploration epochs, this modification allows Sparse-PLS to replace a factor of the actual dimension of the vector with the level of sparsity in the regret bounds, making it smaller (See Theorem~\ref{theorem:sparse_bandits}). The second modification is made to the process to estimate $\theta^*$ at the server. Sparse-PLS employs the LASSO estimator~\cite{Tibshirani1996, Bickel2009} at the server to obtain a sparse estimate of $\theta^*$ from the noisy projections sent by the agents. Please refer to Appendix~\ref{sec:Sparse_PLS_Analysis} for a pseudo-code and additional details about Sparse-PLS. \\

The following theorem characterizes the performance of Sparse-PLS, both in terms of regret and communication cost.
\begin{theorem}
    Consider the distributed stochastic linear bandit setting described in Sec.~\ref{sec:problem_formulation} with an additional assumption of sparsity on the underlying mean reward, i.e., $\|\theta^*\|_0 \leq s$. If Sparse-PLS is run for a time horizon of $T$, then the regret incurred by Sparse-PLS satisfies
\begin{align*}
    R_{\text{Sparse-PLS}}(T) & \leq C \sqrt{sdMT} \log \left(\frac{MK}{\delta} \right) \log\left({\sqrt{MT}}\right),
\end{align*}
with probability at least $1 - \delta$ for some constant $C > 0$, independent of $s, d, M$ and $T$. Moreover, the uplink and downlink communication costs, $C_{\text{u}}(T)$ and $C_{\text{d}}(T)$ are no more than $\cO(s \log T)$ and $\cO(d(\log M + \log T))$ bits respectively.
\label{theorem:sparse_bandits}
\end{theorem}

As it can be noted from the above theorem, Sparse-PLS replaces a factor of dimension $d$ with the sparsity level $s$, or equivalently the effective dimension, in the regret bound matching the optimal-order regret bounds~\cite{AbbasiYadkori12Sparse}. Furthermore, it also reduces $C_{\text{u}}(T)$ from $\cO(d \log T)$ to $\cO(s \log T)$ demonstrating its ability to leverage the inherent sparsity to reduce communication and regret. We refer the reader to Appendix~\ref{sec:Sparse_PLS_Analysis} for a detailed proof.

\section{Conclusion}

In this work, we investigated the communication-learning trade-off in distributed learning setups within the scope of distributed linear bandits. We proposed a novel algorithm, called Progressive Learning and Sharing (PLS), that learns and shares the information about the unknown reward vector progressively, one bit at a time. We showed that PLS incurs order-optimal regret using a uplink communication of $\cO(d \log T)$ bits and downlink communication of $\cO(d(\log M + \log T))$ bits. We also established matching information-theoretic lower bounds on the communication cost for any algorithm with sublinear regret. Lastly, for sparse linear bandits, we showed that a variant of the proposed algorithm offers better communication-learning trade-off by leveraging the sparsity of the problem.

\bibliography{citations}
\bibliographystyle{icml2022}

\newpage

\newpage

\appendix
\onecolumn

\section{Additional Related Work}
\label{sec:related_work_appendix}

In this section, we discuss some more works related to ours broadly in the context of distributed learning. \\

\paragraph{Distributed Bandits:} The multi-armed bandit (MAB) problem with a finite number of arms has been extensively studied in various distributed learning setups. A line of work~\citep{Chawla2020gossip, Landgren2017, Shahrampour2017, Korda2016, Sankararaman2019, Mitra2021, Shi2021, Zhu2021} focuses on developing algorithms for cooperative learning in MAB under different structures of the underlying network. Another line of work~\citep{Liu2010DistributedMAB, Kalathil2012, Rosenski2016, Bistritz2021} explores a collision based approach with no explicit communication inspired by cognitive radio networks. There are several works that also consider the impact of communication and develop communication efficient algorithms by either reducing the frequency of communication~\citep{Agarwal2021, Hillel2013, Tao2019} or proposing quantization approaches~\citep{Hanna2021}. The setup of linear bandits considered in this work is more challenging than the MAB setup considered in the above works, especially for the communication constrained setting. \\

For the problem of distributed linear bandits, there are several other studies in addition to the ones discussed previously.~\cite{Ghosh2021} et al. study the problem under heterogeneity assumptions on the agents and propose a new algorithms under personalization and clustering frameworks, the two different ways adopted to tackle heterogeneity. ~\cite{Chawla2022lowdimensional} consider a high dimensional linear bandit setting where the underlying mean reward lies in a low-dimensional space chosen from a known, finite collection. They propose a decentralized algorithm based on communication over a network to quickly identify this subspace to ensure sub-linear regret. While there is some focus on communication in this study, the proposed algorithm does not offer a linear speed up. Moreover, there is no constraint in terms of bits transmitted over the network.~\cite{Korda2016} et al. consider the problem in a peer-to-peer communication model instead of a star topology. They propose an algorithm that achieves order-optimal regret guarantees with limited communication. However, their proposed policy requires communication of $\cO(d^2)$ bits per message over the network which is worse than the $\cO(d)$ required by PLS. \\

\paragraph{Sparse Linear Bandits:} The problem of sparse linear bandits has been studied mainly in the centralized setting.\cite{AbbasiYadkori12Sparse} proposed an algorithm for sparse linear bandits using a online to batch conversion of the confidence intervals. Borrowing techniques from compressed sensing,~\cite{Carpentier2012BanditCS} proposed an algorithm that incurs an overall regret of $\cO(s\sqrt{T})$.\cite{Chen2022BestSubset} proposed the sparse variants of the famous Lin-UCB~\cite{Abbasi-Yadkori2011LinUCB} and Sup-Lin-UCB~\cite{chu2011contextual} algorithms for the contextual bandit problem.\cite{Hao2020Sparse} proposed an algorithm with dimension-independent regret bounds for very high dimensional problems where the dimension is much larger than the time horizon. As mentioned previously, all these works consider the centralized setting which is different from the distributed setting considered in this work. To the best knowledge of the authors, this is the first work considering sparse linear bandits under a distributed setting with communication constraints. \\

\paragraph{Lower Bounds in Distributed Settings:} Several studies have attempted to characterize information theoretic lower bounds on communication under for various statistical estimation tasks. The classical paper of~\cite{Duchi2014DistributedEstimation} derived guarantees on communication requirements for distributed mean estimation for both independent and interactive protocols. Similarly,~\cite{Tsitsiklis1987} have studied the communication complexity in convex optimization while~\cite{Diakonikolas2017communication} study the effect of communication constraints for the task of distribution estimation. For linear bandits, most of the papers that study lower bounds have been in context of lower bounds on regret~\citep{Dani2008, Rusmevichientong2010} , typically under centralized setting. In the distributed setting, to the best knowledge of the authors, the only work that discusses lower bounds especially in context of communication is~\cite{Amani2022}. They derive a lower bound on the regret for the contextual linear bandit problem under communication constraints. However, in our work we study the lower bounds on communication costs under regret constraints, which is different from the problem considered in their work and requires significantly different analysis techniques.

\section{Analysis for PLS}
\label{sec:PLS_analysis}

Before providing the detailed proofs of the Theorems regarding the performance of PLS, we state and prove two supplementary lemmas that will be used in the proofs.

\begin{lemma}
    Let $\bfX_1, \bfX_2, \dots, \bfX_n$ be a collection of $n$ i.i.d. random vectors in $\R^d$ with mean $\bfmu$. Each coordinate of every random vector is assumed to be an independent sub-Gaussian random variable with variance proxy $\sigma^2$. Let $\{\bfY_i\}_{i = 1}^n$ be another collection of random vectors such that $\bfY_i = \bfX_i \1 \{ x : \|x\| \leq R + B\}$ for all $i \in \{1,2, \dots, n\}$. Then the following relation holds for any $t > 0, R > \sigma\sqrt{2\log(4n)}$ and $B > \|\bfmu\|_{\infty}$,
    \begin{align*}
        \Pr\left(  \left\|\frac{1}{n}\sum_{i = 1}^n\bfY_n - \bfmu\right\|  \geq t\right) \leq 2 \cdot 5^d \cdot \exp\left( -\frac{nt^2}{8\sigma^2} \right).
    \end{align*}
     Consequently, $\displaystyle \left\|\frac{1}{n}\sum_{i = 1}^n\bfY_n - \bfmu\right\| \leq \frac{2\sigma}{\sqrt{n}}(2\sqrt{d} + \sqrt{2\log(2/\delta)})$ with probability at least $1 - \delta$.
     \label{lemma:clipped_subG_rv}
\end{lemma}

\begin{proof}
    Let $\bfv \in \R^d$ be any unit vector. Since the entries of $\bfX_i$ are independent, $Z_i = \bfv^{\top} \bfX$ is a sub-Gaussian random variable with mean $\mu_Z = \bfv^T \bfmu$ and variance proxy $\sigma^2$ for all $i \in \{1,2, \dots, n\}$. Let $W_i = Z_i \1\{[-(R+B), (R+B)]\}$. Since $B \geq \|\bfmu\|_{\infty}$, $\mu_Z \leq B$. We also define the event $\cA := \bigcap_{i = 1}^n \cA_i$ where $\cA_i = \{|Z_i| \leq R+B\}$. For any $\lambda \in \R$, we have,
    \begin{align*}
        \E\left[\exp\left(\lambda\left(\frac{1}{n} \sum_{i = 1}^n W_i - \mu_Z\right)\right)\right] & = \int_{-\infty}^{\infty}  \exp\left(\lambda\left(\frac{1}{n} \sum_{i = 1}^n w_i - \mu_Z\right) \right) f_{W_1, \dots, W_n}(w_1, \dots, w_n) \ \mathrm{d}w_1 \ \dots \ \mathrm{d}w_n \\
        & = \int_{\-\infty}^{\infty}  \exp\left(\lambda\left(\frac{1}{n} \sum_{i = 1}^n w_i - \mu_Z\right) \right) \frac{1}{\Pr(\cA)} f_{Z_1, \dots, Z_n}(w_1, \dots, w_n) \1\{ \cA \} \ \mathrm{d}w_1 \ \dots \ \mathrm{d}w_n \\
        & \leq\int_{\-\infty}^{\infty}  \exp\left(\lambda\left(\frac{1}{n} \sum_{i = 1}^n w_i - \mu_Z\right) \right) \frac{1}{\Pr(\cA)} f_{Z_1, \dots, Z_n}(w_1, \dots, w_n) \ \mathrm{d}w_1 \ \dots \ \mathrm{d}w_n \\
        & \leq  \frac{1}{\Pr(\cA)} \E\left[\exp\left(\lambda\left(\frac{1}{n} \sum_{i = 1}^n Z_i - \mu_Z\right)\right)\right]\\
        & \leq \frac{\exp(\lambda^2 \sigma^2/2n)}{\Pr(\cA)}.
    \end{align*}
    Let us bound the term $\Pr(\cA)$. We have,
    \begin{align*}
        \Pr(\cA) & = \Pr\left(\bigcap_{i = 1}^n \cA_i \right) \\
        & = \Pr\left(\bigcap_{i = 1}^n \{|Z_i| \leq R+B\} \right) \\
        & = 1 - \Pr\left( \bigcup_{i = 1} |Z_i| > R + B \right) \\
        & \geq 1 - n \Pr(|Z_1| > R +B),
    \end{align*}
    Since $Z_1$ is a sub-Gaussian random variable, 
    \begin{align*}
        \Pr(Z_1 > R +B ) & = \Pr(Z_1 - \mu_Z > R + B - \mu_Z) \\
        & \leq \Pr(Z_1 - \mu_Z > R) \\
        & \leq \exp \left( -\frac{R^2}{2 \sigma^2}\right).
    \end{align*}
    Similarly, 
    \begin{align*}
        \Pr(Z_1 < -(R +B) ) & = \Pr(Z_1 - \mu_Z < -R - B - \mu_Z) \\
        & \leq \Pr(Z_1 - \mu_Z < -R) \\
        & \leq \exp \left( -\frac{R^2}{2 \sigma^2}\right).
    \end{align*}
    On combining the two, we obtain,
    \begin{align*}
        \Pr(\cA) \geq 1 - 2n\exp \left( -\frac{R^2}{2 \sigma^2}\right).
    \end{align*}
    If $R \geq \sigma\sqrt{2 \log(4n)}$, then $\Pr(\cA) \geq 1/2$. Consequently, 
    \begin{align*}
        \E\left[\exp\left(\lambda\left(\frac{1}{n} \sum_{i = 1}^n W_i - \mu_Z\right)\right)\right] \leq 2\exp(\lambda^2 \sigma^2/2n),
    \end{align*}
    and
    \begin{align*}
        \Pr\left(\frac{1}{n} \sum_{i = 1}^n W_i - \mu_Z > t \right) \leq 2 \exp(-nt^2/2\sigma^2). 
    \end{align*}
    To obtain concentration bounds for $\displaystyle \left\|\frac{1}{n}\sum_{i = 1}^n\bfY_n - \bfmu\right\|$, we use the same technique used for unbounded sub-Gaussian random variables. We reproduce the proof here for completeness. For brevity of notation, we define $\displaystyle \bfW := \frac{1}{n}\sum_{i = 1}^n\bfY_n$. \\
    
    Let $\cN$ denote a minimal $1/2$-cover of $\cB_{d}(1)$, where $\cB_d(r) = \{\bfx \in \R^d : \|\bfx\|_2 \leq r\}$. This implies that for all $\bfx \in \cB_d(1)$, $\exists \bfy \in \cN$ such that $\|\bfx - \bfy \|_2 \leq 1/2$. Using the standard volumetric bounds, the cardinality of $\cN$ can be bounded as $|\cN| \leq 5^d$. \\

    Once again, let $\bfv \in \cB_d(1)$ denote a unit vector. For every $\bfv$, we have a $\bfz_{\bfv} \in \cN$ such that $\bfv = \bfz_{\bfv} + \bfu$ with $\|\bfu\| \leq 1/2$. Thus, for any vector $\bfX$, we have,
    \begin{align*}
        \max_{\bfv \in \cB_d(1)} \bfv^{\top} \bfX & = \max_{\bfv \in \cB_d(1)} (\bfz_{\bfv} + \bfu)^{\top} \bfX \\
        & \leq \max_{\bfz \in \cN} \bfz^{\top}\bfX + \max_{\bfu \in \cB_d(1/2)} \bfu^{\top}\bfX \\
        & \leq \max_{\bfz \in \cN} \bfz^{\top}\bfX + \frac{1}{2} \max_{\bfu \in \cB_d(1)} \bfu^{\top}\bfX.
    \end{align*}
    Thus, $\max_{\bfv \in \cB_d(1)} \bfv^{\top} \bfX \leq 2\max_{\bfz \in \cN} \bfz^{\top}\bfX$. Moreover, since $\|\bfv\|_2 = 1$, $\max_{\bfv \in \cB_d(1)} \bfv^{\top} \bfX = \|\bfX\|$. Consequently, 
    \begin{align*}
        \Pr\left(  \left\|\bfZ - \bfmu\right\|  \geq t\right) & \leq \Pr( \max_{\bfv \in \cB_d(1)} \bfv^{\top} (\bfW - \bfmu) \geq t) \\
        & \leq \Pr( \max_{\bfz \in \cN} \bfz^{\top}(\bfW - \bfmu) \geq t/2) \\
        & \leq 2|\cN| \exp\left( -\frac{nt^2}{8\sigma^2} \right).
    \end{align*}
    Plugging in the value of $|\cN|$ yields the final result.
\end{proof}

\begin{lemma}
    For any epoch $k$ in PLS, the estimate $\hat{\theta}_{k}^{(\textsc{serv})}$ satisfies
    \begin{align*}
        \|\hat{\theta}_{k}^{(\textsc{serv})} - \theta^*\| \leq \tau_k,
    \end{align*}
    with probability at least $1 - \delta$.
    \label{lemma:theta_estimate_error}
\end{lemma}

\begin{proof}

We prove the claim using induction. Firstly, note that if we define $\bar{\theta}_{k-1} := 0$ for all epochs $k$ during the norm estimation stage, then we can rewrite the method to compute the estimate at the server, $\hat{\theta}_{k}^{(\textsc{serv})}$, during the norm estimation stage in the same way as it is computed during the refinement stage. Thus, we consider the definition of $\hat{\theta}_{k}^{(\textsc{serv})}$ as used in Algorithm~\ref{alg:ref_est_server} while implicitly assuming $\bar{\theta}_{k-1} = 0$ for all epochs $k$ during the norm estimation stage. \\

Let $\eta_k^{(j)} \in \R^d$ denote the quantization noise added by $j^{\text{th}}$ client during the $k^{\text{th}}$ epoch, i.e., the quantized version received by the server can be written as $Q(\tilde{\theta}^{(j)}_k) = \tilde{\theta}^{(j)}_k + \eta_k^{(j)}$. Since each coordinate is quantized independently, each coordinate of $\eta_k^{(j)}$ is an independent zero mean sub-Gaussian random variable with variance proxy $\alpha_k^2/4d$. At the end of the $k^{\text{th}}$ epoch, we have
\begin{align*}
    \|\hat{\theta}_{k}^{(\textsc{serv})} - \theta^*\| & = \left\|\bar{\theta}_{k-1} + \frac{1}{M} \sum_{j = 1}^M Q(\tilde{\theta}^{(j)}_k) - \theta^* \right\| \\
    & = \left\| \bar{\theta}_{k-1} + \frac{1}{M} \sum_{j = 1}^M (\tilde{\theta}^{(j)}_k + \eta_k^{(j)}) - \theta^* \right\| \\
    & = \left\| \bar{\theta}_{k-1} + \frac{1}{M} \sum_{j = 1}^M(\hat{\theta}^{(j)}_k - \bar{\theta}_{k-1})\1_{R_k + B_k} + \frac{1}{M} \sum_{j = 1}^M \eta_k^{(j)} - \theta^* \right\| \\
    & \leq \left\|  \frac{1}{M} \sum_{j = 1}^M(\hat{\theta}^{(j)}_k - \bar{\theta}_{k-1})\1_{R_k + B_k} - (\theta^* - \bar{\theta}_{k-1}) \right\|   + \left\| \frac{1}{M} \sum_{j = 1}^M \eta_k^{(j)}  \right\|.
\end{align*}

The second term can be bounded using the concentration of sub-Gaussian random vectors as
\begin{align*}
    \left\| \frac{1}{M} \sum_{j = 1}^M \eta_k^{(j)}  \right\| & \leq \frac{2\alpha_k}{\sqrt{M}}\left(1 + \sqrt{\frac{1}{2d}\log\left(\frac{2K}{\delta}\right)} \right).
\end{align*}
For the first term, we will employ Lemma~\ref{lemma:clipped_subG_rv}. However, we first need to ensure that the conditions of the lemma are satisfied. For any epoch $k$, the particular choice of $R_k$ and $s_k$ in PLS satisfies the assumption in Lemma~\ref{lemma:clipped_subG_rv} for all $k$.  To bound the mean, we need to consider the epochs with $\bar{\theta}_{k-1} = 0$ (i.e., the norm estimation stage) and $\bar{\theta}_{k-1} \neq 0$ (the refinement stage) separately. We begin with the case of $\bar{\theta}_{k-1} = 0$. Under this scenario, note that $ \| \E[\hat{\theta}^{(j)}_k - \bar{\theta}_{k-1}] \| = \| \E[\hat{\theta}^{(j)}_k] \| = \|\theta^* \|$. For the first epoch, we know that $\|\theta^*\| \leq 1 = B_1$. This implies, we can apply the lemma for the base case. For any epoch $k \geq 2$, we use the inductive hypothesis to establish the bound on the mean. In particular, we use a better estimate of $\|\theta^*\|$ by noting that the norm estimation stage had not been terminated by epoch $k - 1$ which implies that $\|\hat{\theta}_{k-1}^{(\textsc{serv})}\| \leq 4\tau_{k - 1}$. Along with the induction hypothesis, $\|\hat{\theta}_{k-1}^{(\textsc{serv})} - \theta^*\| \leq \tau_{k-1}$, we can conclude that $\|\theta^*\| \leq 5\tau_{k-1} \leq B_k$, as required. \\

Now we are ready to invoke Lemma~\ref{lemma:clipped_subG_rv}. Using Lemma~\ref{lemma:clipped_subG_rv}, we can conclude that with probability at least $1 - \delta/K$,
\begin{align*}
    \left\|  \frac{1}{M} \sum_{j = 1}^M(\hat{\theta}^{(j)}_k - \bar{\theta}_{k-1})\1_{R_k + B_k} - (\theta^* - \bar{\theta}_{k-1}) \right\|   & \leq \frac{4\sigma\sqrt{d}}{\sqrt{Ms_k}}\left(1 +  \sqrt{\frac{1}{2d}\log\left(\frac{2K}{\delta}\right)} \right).
\end{align*}

On plugging in the value of $s_k$ and $\alpha_k$ and combining the equations, we obtain,
\begin{align*}
    \|\hat{\theta}_{k}^{(\textsc{serv})} - \theta^*\| \leq \frac{2^{-k}}{\sqrt{M}} + \frac{2^{-(k+1)}}{\sqrt{M}}  = \frac{3}{\sqrt{M}} \cdot 2^{-(k+1)} = \tau_k,
\end{align*}
holds with probability at least $1- \delta/K$. \\

We similarly consider the case where in epoch $k$, $\bar{\theta}_{k-1} \neq 0$. For this case, we apply Lemma~\ref{lemma:clipped_subG_rv} conditioned on the value of $\bar{\theta}_{k-1}$ and hence all expectations and probabilities are computed with respect to the conditional measure. For brevity of notation and ease of understanding, we will assume the conditioning has been implicitly applied. The condition on $R_k$ holds using the argument as in the previous case and hence we focus only on showing the bound on the mean. If $\zeta_{k-1}$ denotes the quantization error during the $k - 1$ epoch, then we can write $\bar{\theta}_{k-1} = \hat{\theta}_{k-1}^{(\textsc{serv})} + \zeta_{k-1}$. Consequently, $ \| \E[\hat{\theta}^{(j)}_k - \bar{\theta}_{k-1}] \| = \|\theta^* - \bar{\theta}_{k-1} \| = \|\theta^* - \hat{\theta}_{k-1}^{(\textsc{serv})} - \zeta_{k-1}\| \leq  \|\theta^* - \hat{\theta}_{k-1}^{(\textsc{serv})}\| + \|\zeta_{k-1}\| \leq \|\theta^* - \hat{\theta}_{k-1}^{(\textsc{serv})}\| + \beta_{k-1} \leq \tau_{k - 1} + \beta_{k-1} \leq B_k$. For the last step, the bound on $\|\theta^* - \hat{\theta}_{k-1}^{(\textsc{serv})}\|$ holds due to the hypothesis in the induction step. In the base case, i.e., $k_0$, the index of the epoch when the norm estimation stage terminates, the bound holds from the result obtained for the case of $\bar{\theta}_{k-1} = 0$. This implies that we can apply the lemma also for the case of $\bar{\theta}_{k-1} \neq 0$, thereby obtaining the bound on $\|\hat{\theta}_{k}^{(\textsc{serv})} - \theta^*\|$ for all $k$. \\

We would like to point that to avoid repeating steps in the proof, we have directly shown the result after invoking Lemma~\ref{lemma:clipped_subG_rv} for all epochs $k$. The proof actually follows by first invoking Lemma~\ref{lemma:clipped_subG_rv} for the base case and establishing the result. For any epoch $k$, we then use the inductive hypothesis of the relation for $k - 1$ to establish the conditions on the mean after which Lemma~\ref{lemma:clipped_subG_rv} is invoked to complete the induction step. \\

Lastly, consider the events given by $\cE_{k} := \{ \|\hat{\theta}_{k}^{(\textsc{serv})} - \theta^*\| \leq \tau_k\} $ for $k = 1, 2, \dots K$ and the event $\displaystyle \cE := \bigcap_{k =1}^{K} \cE_k$. From the above analysis, we know that $\Pr(\cE_k|\cE_1, \cE_2, \dots, \cE_{k-1}) \geq 1 - \delta/K$ for $k = 1, 2, \dots, K$, where $\Pr(\cE_0) = 1$. Thus, we have
\begin{align*}
    \Pr(\cE) & = \Pr\left( \bigcap_{k =1}^{K} \cE_k\right) \\
    & = \prod_{i = 1}^k \Pr(\cE_k|\cE_1, \cE_2, \dots, \cE_{k-1}) \\
    & \geq \prod_{i = 1}^K \left( 1- \frac{\delta}{K} \right)^K \\
    & \geq 1 - \delta,
\end{align*}
as required.
\end{proof}

We now proceed to provide detailed proofs of the various theorems and lemmas that characterize the regret and communication cost of PLS.

\subsection{Proof of Lemma~\ref{lemma:regret_norm_est}}
\label{proof:regret_norm_est}

Let $R_{\textsc{NE}}(T)$ denote the regret incurred during the norm estimation stage. We assume that the stage continues until a \texttt{terminate} is received from the server or each client has played a total of $T$ actions, whichever happens first. \\

Under the event $\cE$ as defined in Proof of Lemma~\ref{lemma:theta_estimate_error}, the algorithm will terminate at the end of epoch $k_0$ where $k_0 := \min \{ k \in \N: 4\tau_k \leq \| \hat{\theta}_{k}^{(\textsc{serv})} \| \}$. We note that for all $k$ for which the inequality $4\tau_k \leq \| \hat{\theta}_{k}^{(\textsc{serv})}\|$ holds, we also have the relation
\begin{align*}
     \|\hat{\theta}_{k}^{(\textsc{serv})} - \theta^*\| &\leq \tau_k \leq \frac{1}{4} \cdot \| \hat{\theta}_{k}^{(\textsc{serv})} \|  \\
     \implies \|\theta^*\| &\in \left[ \frac{3}{4} \| \hat{\theta}_{k}^{(\textsc{serv})} \|, \frac{5}{4} \| \hat{\theta}_{k}^{(\textsc{serv})} \| \right] \\
     \implies \| \hat{\theta}_{k}^{(\textsc{serv})} \| &\in \left[ \frac{4}{5} \|\theta^*\|, \frac{4}{3} \| \theta^* \| \right].
\end{align*}
Consequently, we also have $\tau_k \leq \frac{1}{3} \cdot \| \theta^* \|$. On plugging in the value of $\tau_k$, we obtain
\begin{align*}
    \tau_k &\leq \frac{1}{3} \cdot \| \theta^* \| \implies 2^{-k} \leq \frac{2\sqrt{M}}{9} \|\theta^*\| \implies k \geq \log_2 \left(\frac{9}{2\|\theta^*\|\sqrt{M}}  \right).
\end{align*}
Since $k_0$ is the smallest natural number satisfying this relation, $\displaystyle k_0 = \max\left\{\left\lceil\log_2 \left(\frac{9}{2\|\theta^*\|\sqrt{M}}  \right)\right\rceil, 1 \right\}$. \\

We know that $a_t, a^* \in  \cA = \{x: \|x\|_2 \leq 1\}$ for all $t \in \{1, 2, 3, \dots, T\}$. Hence, the instantaneous regret at time instant can be bounded as 
\begin{align*}
    r_t = \ip{\theta^*}{a^*} - \ip{\theta^*}{a_t} \leq \|\theta^*\|_2 \|a^* - a_t\|_2 \leq 2 \|\theta^*\|_2.
\end{align*}
Consequently, we can bound $R_{\textsc{NE}}(T)$ as follows. If $k_0 = 1$, then $R_{\textsc{NE}}(T)$ can be simply upper bounded as $Mds_1 = \cO(1)$. If $k_0 \geq 2$, we have
\begin{align*}
    R_{\textsc{NE}}(T) & \leq \sum_{k = 1}^{k_0} 2M\|\theta^*\|_2 \cdot ds_k \\
    & \leq 2Md\|\theta^*\|_2 \cdot s_{k_0} \cdot k_0 \\
    & \leq 2Md\|\theta^*\|_2 \cdot \left(32d\sigma^2 \log\left(\frac{8MK}{\delta} \right) 4^{\log_2(9/2\|\theta^*\|_2) + 1} + 1 \right) \cdot \left( \log_2 \left(\frac{9}{2\|\theta^*\|_2}  \right) + 1\right) \\
    & \leq 2Md\|\theta^*\|_2 \cdot \left(2592\frac{d\sigma^2}{\|\theta^*\|_2^2} \log\left(\frac{8MK}{\delta} \right)  + 1 \right) \cdot \left( \log_2 \left(\frac{9}{2\|\theta^*\|_2}  \right) + 1\right) \\
    & \leq 5184 \frac{d^2\sigma^2}{\|\theta^*\|_2} \log\left(\frac{8MK}{\delta} \right)\left( \log_2 \left(\frac{9}{2\|\theta^*\|_2}  \right) + 1\right) + 2d \left( \log_2 \left(\frac{9}{2\|\theta^*\|_2}  \right) + 1\right).
\end{align*}
A trivial bound on $R_{\textsc{NE}}(T)$ is $2\|\theta^*\|_2 MT$. Note that for larger values of $\|\theta^*\|_2$, the former bound is tighter. However, as $\|\theta^*\|_2$ gets smaller, the latter bound becomes a stronger one. To obtain the optimal bound on $R_{\textsc{NE}}(T)$, we choose the minimum of the two. Thus,
\begin{align*}
    R_{\textsc{NE}}(T) \leq \min \left\{5184 \frac{d^2\sigma^2}{\|\theta^*\|_2} \log\left(\frac{8MK}{\delta} \right)\left( \log_2 \left(\frac{9}{2\|\theta^*\|_2}  \right) + 1\right) + 2d \left( \log_2 \left(\frac{9}{2\|\theta^*\|_2}  \right) + 1\right), 2\|\theta^*\|_2 MT \right\}.
\end{align*}
On setting $\|\theta^*\|_2 = d/\sqrt{MT}$, we obtain $R_{\textsc{NE}}(T)$ is $\cO(d\sqrt{MT} \cdot \log(MT) \cdot \log(\log T/\delta))$, as required.

\subsection{Proof of Lemma~\ref{lemma:theta_error_regret}}
\label{proof:theta_error_regret}

We are given access to estimate, $\hat{\theta}$, of the true vector, $\theta$, such that $\|\hat{\theta} -\theta\|_2 \leq \tau \leq \|\theta\|_2$. This implies that $\|\hat{\theta}\|_2 \in [\|\theta\|_2 -\tau, \|\theta\|_2 + \tau]$. Consider the relation,
\begin{align*}
    \|\hat{\theta} - \theta\|_2^2 & \leq \tau^2 \\
    \implies \|\hat{\theta}\|_2^2 + \|\theta^*\|_2^2 - 2 \ip{\hat{\theta}}{\theta} & \leq \tau^2 \\
    \implies -\ip{\hat{\theta}}{\theta} & \leq \frac{1}{2} \left(\tau^2 - \|\hat{\theta}\|_2^2 - \|\theta\|_2^2 \right) \\
    \implies \|\theta\| - \frac{1}{\|\hat{\theta}\|_2}\ip{\hat{\theta}}{\theta} & \leq \frac{1}{2\|\hat{\theta}\|_2} \left(\tau^2 - \|\theta\|_2^2 - \|\theta\|_2^2 + 2\|\theta\|_2 \|\theta\|_2\right) \\
    & \implies \|\theta\| - \frac{1}{\|\hat{\theta}\|_2}\ip{\hat{\theta}}{\theta} \leq \frac{1}{2\|\theta\|_2} \left(\tau^2 - (\|\theta\|_2 - \|\theta\|_2)^2 \right).
\end{align*}
Note that the LHS in the last equation is an upper bound on the regret incurred by playing the action $a = \hat{\theta}/\|\hat{\theta}\|_2$. If we denote the regret incurred by playing the action $a$ by $\text{Reg}(a)$ and let $\|\hat{\theta}\|_2 = \|\theta\|_2 + \upsilon$, for some $\upsilon \in [-\tau, \tau]$, then 
\begin{align*}
    \text{Reg}(a) & \leq \frac{\tau^2 - \upsilon^2}{2(\|\theta\|_2 + \upsilon)}.
\end{align*}
To bound the above expression, we consider the function $f(\upsilon) = \dfrac{\tau^2 - \upsilon^2}{(\|\theta\|_2 + \upsilon)}$. Since $f$ is rational function of two polynomials, it is differentiable. On differentiating $f$, we obtain $f'(\upsilon) = -\dfrac{\upsilon^2 + 2\upsilon\|\theta\|_2 + \tau^2}{(\|\theta\|_2 + \upsilon)^2}$. The solutions for $f'(\upsilon) = 0$ are given by $\upsilon_{+} = -\|\theta\|_2 + \sqrt{\|\theta\|_2^2 - \tau^2}$ and $\upsilon_{-} = -\|\theta\|_2 - \sqrt{\|\theta\|_2^2 - \tau^2}$. By double differentiating $f$, we can verify that it is indeed maximized at $\upsilon_{+}$ for $\upsilon \in [-\tau, \tau]$. Moreover, note that solutions for $f'(\upsilon) = 0$ are real numbers only for $\tau \leq \|\theta\|_2$. This explains the need for the constraint on $\tau$. On setting $\upsilon = \upsilon_{+}$ in the expression for $\text{Reg}(a)$, we obtain the following bound.
\begin{align*}
    \text{Reg}(a) & \leq \frac{\tau^2 - (-\|\theta^*\|_2 + \sqrt{\|\theta\|_2^2 - \tau^2})^2}{2\sqrt{\|\theta\|_2 - \tau^2}} \\
    & \leq \frac{\tau^2 - \|\theta\|_2^2 - \|\theta\|_2^2 + \tau^2 + 2\|\theta\|_2\sqrt{\|\theta\|_2^2 - \tau^2}}{2\sqrt{\|\theta\|_2 - \tau^2}} \\
    & \leq \|\theta\|_2 - \sqrt{\|\theta\|_2^2 - \tau^2}  \\
    & \leq \frac{\tau^2}{\|\theta\|}.
\end{align*}

\subsection{Proof of Theorem~\ref{theorem:regret_bound}}
\label{proof:regret_bound}

Let $R_{\textsc{REF}}(T)$ denote the regret incurred during the refinement stage. To obtain a bound on $R_{\textsc{REF}}(T)$, we consider an epoch $k$ during the refinement stage. Similar to the norm estimation stage, we bound the regret incurred during the exploration sub-epoch as $ 2\|\theta^*\|_2 \cdot Mds_k $. Using Lemma~\ref{lemma:theta_error_regret} and~\ref{lemma:theta_estimate_error}, we can bound the regret during the exploitation sub-epoch as $ \dfrac{\tau_k^2}{\|\theta^*\|_2} \cdot Mt_k$. If $R^{(k)}$ denotes the regret incurred during epoch $k$, then we have the following relation.
\begin{align*}
    R^{(k)} & \leq 2\|\theta^*\|_2 \cdot Mds_k + \frac{M\tau_k^2}{\|\theta^*\|_2} \cdot (Ms_k^2\mu_0^2 + 1) \\
    & \leq 2\|\theta^*\|_2 \cdot Mds_k + \frac{16(M\tau_k s_k \|\theta^*\|_2)^2}{9\|\theta^*\|_2}  + \frac{M\tau_k^2}{\|\theta^*\|_2} \\
    & \leq 64\|\theta^*\|_2 Md^2\sigma^2 \log\left(\frac{8MK}{\delta} \right) 4^{k} + 2\|\theta^*\|_2 \cdot Md + 8M\|\theta^*\|_2 \left(1024 \sigma^4 d^2  \log^2\left(\frac{8MK}{\delta} \right) 4^{k} + 4^{-k}\right) + M \|\theta^*\|_2 \\
    & \leq 72  \cdot \|\theta^*\|_2 Md^2\sigma^2(1 + \sigma^2)\log^2\left(\frac{8MK}{\delta} \right) 4^{k} + M\|\theta^*\|_2 \left( 8 \cdot 4^{-k} + 2d + 1\right)
\end{align*}
Note that the regret incurred during the exploration sub-epoch is the same as that incurred during the exploitation sub-epoch upto constant factors. This echoes the discussion in Sec.~\ref{sec:analysis} on the careful choice of the lengths of exploration and exploitation epochs in PLS using the estimated norm of $\|\theta^*\|_2$. \\

Let $k_1$ denote the index of the epoch when the query budget ends. Also, recall that $k_0$ was defined to be the epoch index during which the norm estimation stage terminates. If $k_1 = k_0$, then it implies that the exploitation sub-epoch of the $k_0^{\text{th}}$ epoch was not completed. Consequently, the regret incurred during the partial sub-epoch is bounded by the regret incurred during the corresponding exploration sub-epoch (upto a constant factor) which in turn is bounded by the regret incurred during the norm estimation stage. Hence, the overall regret has the same order as the of the norm estimation stage, that is, $\tilde{\cO}(d\sqrt{MT})$, as required. For the rest of the proof we focus on the case $k_1 \geq k_0 + 1$. \\

The regret incurred during the refinement stage can be be bounded as
\begin{align*}
    R_{\textsc{REF}}(T) \leq \sum_{k = k_0}^{k_1} R^{(k)}.
\end{align*}
Since we already have a bound on $k_0$, we can bound the above expression by finding an upper bound on $k_1$. We can bound $k_1$ by using the length of the time horizon. We have,
\begin{align*}
    \sum_{k = 1}^{k_1-1} ds_k + \sum_{k = k_0}^{k_1-1} t_k \leq T.
\end{align*}
The first term corresponds to the length of all the exploration (sub-)epochs, including the ones during the norm estimation procedure, while the second accounts for the length of all the exploitation sequences. On plugging in the values of $s_k$ and $t_k$, we obtain,
\begin{align*}
    T & \geq \sum_{k = 1}^{k_1-1} ds_k + \sum_{k = k_0}^{k_1-1} t_k \\
    & \geq \sum_{k = k_0}^{k_1-1} Ms_k^2 \mu_0^2 \\
    & \geq \sum_{k = k_0}^{k_1-1} \frac{16384}{25} Md^2 \|\theta^*\|^2 \sigma^4 \log^2\left(\frac{8MK}{\delta} \right) 16^{k}   \\
    & \geq \frac{16384}{25} Md^2  \|\theta^*\|^2 \sigma^4 \log^2\left(\frac{8MK}{\delta} \right) \frac{16^{k_1} - 16^{k_0}}{15}  \\
    & \geq 40 Md^2   \|\theta^*\|^2 \sigma^4 \log^2\left(\frac{8MK}{\delta} \right) 16^{k_1}.
\end{align*}
where the last step follows by noting that $k_1 \geq k_0 + 1$. This implies that $\displaystyle k_1 \leq \log_{16} \left( \frac{T}{40 M d^2  \|\theta^*\|_2^2 \sigma^4} \left(\log\left(\frac{8MK}{\delta} \right)\right)^{-2}\right)$. Consequently, 
\begin{align*}
    R_{\textsc{REF}}(T) & \leq \sum_{k = k_0}^{k_1} R^{(k)} \\
    & \leq \sum_{k = k_0}^{k_1} 72  \cdot \|\theta^*\|_2 Md^2\sigma^2(1 + \sigma^2)\log^2\left(\frac{8MK}{\delta} \right) 4^{k} + M\|\theta^*\|_2 \left( 8 \cdot 4^{-k} + 2d + 1\right) \\
    & \leq  96  \cdot \|\theta^*\|_2 Md^2\sigma^2(1 + \sigma^2)\log^2\left(\frac{8MK}{\delta} \right) 4^{k_1} + Mk_1\|\theta^*\|_2 \left( 2d + 9\right) \\
    & \leq  96  \cdot \|\theta^*\|_2 Md^2\sigma^2(1 + \sigma^2)\log^2\left(\frac{8MK}{\delta} \right) \cdot \frac{1}{d\sigma^2\|\theta^*\|_2} \left(\log\left(\frac{8MK}{\delta} \right)\right)^{-1} \cdot \sqrt{\frac{T}{40M}} + Mk_1\|\theta^*\|_2 \left( 2d + 9\right) \\
    & \leq  16 (1 + \sigma^2)\log\left(\frac{8MK}{\delta} \right)  \cdot d\sqrt{MT} + Mk_1\|\theta^*\|_2 \left(  2d + 9\right).
\end{align*}

Adding this to the regret incurred during the norm estimation stage, we can conclude that $R_{\text{PLS}}(T)$ satisfies $\cO(d\sqrt{MT} \cdot \log(MT) \cdot \log(\log T/\delta))$.

\subsection{Proof of Lemma~\ref{lemma:channel_capacity}}
\label{proof:channel_capacity}

Consider a vector $x \in \R^d$ with $\|x\| \leq r$ and let $Q(x)$ denote its quantized version achieved by a quantizer (deterministic or stochastic) upto a precision of $\varepsilon$. This implies that $\|x - Q(x)\|_2 \leq \varepsilon \implies \|Q(x)\|_2 \leq r + \varepsilon$. Note that $Q(x)$ can be represented as $(q_1, q_2, \dots, q_d)^{\top}$ where $q_i \in \{-l_{\varepsilon}/2, \dots, 0, \dots, l_{\varepsilon}/2\}$ represents the corresponding quantized index for all $i \in \{1,2, \dots, d\}$. Thus, we have,
\begin{align*}
    \sum_{i = 1}^d q_i^2 \left( \frac{2r}{l_{\varepsilon}}\right)^2 & \leq (r + \varepsilon)^2 \\
    \implies \sum_{i = 1}^d q_i^2 & \leq \left(\frac{(r + \varepsilon) l_{\varepsilon}}{2r} \right)^2 \\
    & \leq 4d \left(\frac{r}{\varepsilon} + 1 \right)^2.
\end{align*}
Consequently, 
\begin{align*}
    \sum_{i = 1}^d |q_i| \leq \sqrt{d\sum_{i = 1}^d q_i^2} \leq 2d \left(\frac{r}{\varepsilon} + 1 \right).
\end{align*}

It can be noted that under the encoding scheme based on unary encoding used in PLS, the message size in bits in PLS is given by $3d + \sum_{i = 1}^d |q_i|$, where $3d$ corresponds to the sum of lengths of the headers. As a result, the message size is bounded by $d (3 + 2(r/\varepsilon + 1))$. We use this relation to establish the message sizes on both the uplink and the downlink channels. \\

Let us first consider the uplink communication. In epoch $k$, $r$ corresponds to $R_k + B_k$ and $\varepsilon$ to $\alpha_k$. On plugging in the prescribed values of the above parameters, we note that $(R_k + B_k)/\alpha_k$ is $C/\alpha_0$, where $C$ is a constant independent of $d, M$ and $T$. Hence, any message sent on the uplink channel is no more than $\cO(d)$ bits. Similarly, for the downlink communication, the ratio $(B_k + \tau_k)/\beta_k \leq C'/\beta_0$ where once again $C'$ is a constant independent of $d, M$ and $T$. Hence, any message sent on the downlink channel also has a size of $\cO(d)$ bits, as required.

\subsection{Proof of Theorem~\ref{theorem:comm_bound}}
\label{proof:comm_bound}

The bound on the uplink communication cost follows immediately from Lemma~\ref{lemma:channel_capacity}. Since the agents send a message of $\cO(d/\alpha_0)$ bits only once every epoch and there are no more than $K = \cO(\log T)$ epochs in PLS, the uplink communication cost of PLS, $C_{\text{u}}(T)$, satisfies $\cO((d/\alpha_0) \log T)$. The $\cO((d/\beta_0)\log T)$ term in $C_{\text{d}}(T)$, the downlink cost, also follows using the same argument.  \\

The $\cO(d \log M)$ term in the downlink communication cost corresponds to sending the initial estimate of $\theta^*$ to the clients, specifically in the event that the norm estimation stage ends within the first epoch. Note that, if the norm estimation stage ends in the first epoch, the estimation error in $\theta^*$ becomes $\tau_1 = \cO(\sqrt{1/M})$ from the initial estimate of $\cO(1)$. As a result, PLS requires $\cO(d\log M)$ bits to transmit the estimate, adding to the downlink communication cost. This $\cO(d \log M)$ cost is what facilitates information exchange between the agents that leads to the linear speedup with respect to the number of agents.

\begin{remark}
    This estimate is also sent using the unary encoding scheme used in PLS over $\cO(\log M)$ rounds with each round sending one additional bit of information per coordinate. Depending on the synchronization requirements as dictated by the hardware, the learner may not be allowed multiple uses of the channel. In the event that it is possible to use the channel several times between two actions, this information can be easily sent with $\cO(\log M)$ uses of the channel. However, if the server is allowed to use the channel only once between two actions of the agent, this transmission of the initial estimate can be accommodated within PLS as follows. At the end of the exploration sub-epoch of the first epoch, instead of starting with the exploitation sub-epoch, the agents begin with the exploration sub-epoch of the second epoch. During this time, the server broadcasts the initial estimate of $\theta^*$ over the next $\cO(\log M)$ time instants. The agents continue to play the actions as dictated by the exploration sub-epoch until the communication is completed. Once the communication is completed, the agents now start with the exploitation sub-epoch of the first epoch. At the end of the first epoch, they restart the exploration sub-epoch, starting from where they had left at the end of the communication. Thus, if needed, PLS can accommodate the additional transmission requirements by a minor reordering of the explorative and exploitative actions.
\end{remark}

\section{Lower Bounds}
\label{sec:lower_bound_proofs}

In this section, we discuss the details of the proofs to establish the lower bounds on the communication cost under regret constraints.

\subsection{Proof of Theorem~\ref{theorem:comm_lower_bound}}

The proof of this theorem consists of two main steps. In the first step, we show that all algorithms achieving a sub-linear cumulative regret need to solve the problem of distributed mean estimation. This reduction allows us to leverage several existing techniques and results for information-theoretic lower bounds, especially in the case of distributed statistical estimation. The second step is to establish lower bound on communication cost (both uplink and downlink) for a given estimation error based on the above reduction. The primary idea for this step is to use to classical reduction to identification from a specifically constructed hard instance. Specifically, we show that for any estimator with a small error it is necessary to solve the identification problem. The final bound on the communication cost is then obtained by using Fano's inequality to bound the error of this identification problem. \\

We first establish how the constraint of a sub-linear cumulative regret for linear bandit algorithm can be translated to having a small simple regret.Consider any policy $\pi$ that achieves a sub-linear regret $R_{\pi}(T)$ under the setup described in Sec.~\ref{sec:problem_formulation}. Consequently, the policy $\pi$ also guarantees a sub-linear regret of $R_{\pi}(T)/MT$. This follows immediately by choosing the final action to be the average of all the actions chosen by all the agents. Note that this is a permissible action since $\cA$ is a convex set. As a result, a lower bound of $\underline{r}(T)$ on simple regret achievable by any policy immediately implies a lower bound of $\underline{R}(T) = MT\underline{r}(T)$ on the cumulative regret achievable by any policy. This relation is used later in the proof to draw parallels with the problem of distributed mean estimation. \\

For the second step, we separately establish the lower bounds for the uplink and downlink communication cost , by constructing two different hard instances. We begin with the lower bound on the downlink cost.

\subsubsection{Proof of downlink cost}

For the lower bound on the downlink cost, recall that for a policy a $\pi$ with sub-linear cumulative regret of $R_{\pi}(T)$, the average of all the actions chosen by all the agents, denoted by $\bar{a}_{\pi}$, achieves a simple regret of $R_{\pi}(T)/MT$. This implies that using the actions taken by $\pi$, one can estimate $\theta^*$ to a reasonable accuracy dictated by $R_{\pi}(T)$. In particular, let $\hat{\theta}(A; \pi)$ denote an estimator of $\theta^*$ based on all the actions taken by a policy $\pi$, which we denote by the random variable $A$, and $\|\theta^*\|_2 = 1$. Then,
\begin{align*}
    \inf_{\hat{\theta}}\|\hat{\theta}(A; \pi) - \theta^*\|_2^2 & \leq \left\| \frac{\bar{a}_{\pi}}{\|\bar{a}_{\pi}\|_2} - \theta^*\right\|_2^2 \\
    & \leq \frac{\|\bar{a}_{\pi}\|_2^2}{\|\bar{a}_{\pi}\|_2^2} + \|\theta^*\|_2^2 - 2 \ip{\theta^*}{\frac{\bar{a}_{\pi}}{\|\bar{a}_{\pi}\|_2}} \\
    & \leq 2(1  -  \ip{\theta^*}{\bar{a}_{\pi}} )\\
    & \leq 2 \frac{R_{\pi}(T)}{MT}.
\end{align*}
We use the above relation to obtain a bound on the downlink communication cost based on the information carried in $A$, the actions taken by a policy, about the underlying vector $\theta^*$. \\ 

Let $\cV_{\varepsilon}$ denote a maximal $2\varepsilon$-packing of $\cS^{d}$ using the $\|\cdot\|_2$ norm, where $\cS^d$ denotes the surface of a unit sphere in $\R^d$. Equivalently, $\cS^d = \partial \cB_d(1) \in \R^{d-1}$, where $\cB_d(r)$ denotes a ball (in $\|\cdot\|_2$-norm) of radius $r$ in $\R^d$. Let $V$ be a random variable chosen uniformly at random from $\cV_{\varepsilon}$. Consider a linear bandit instance instantiated with $\theta^* = V$. Note that this construction implicitly ensures $\|\theta^*\|_2 = 1$. As before, let $\hat{\theta}(A; \pi)$ denote an estimator of $\theta^*$ based on all the actions taken by a policy $\pi$. This problem of estimating $\theta^*$ can be mapped to that of identifying $V$ using classical techniques by defining a testing function $\hat{V}$ for each estimator $\hat{\theta}$. In particular, define $\hat{V} := \argmin_{v \in \cV_{\varepsilon}}\|\hat{\theta}(A; \pi) - v\|_2$, that is, it maps $\hat{\theta}$ to the closest point in the set $\cV_{\varepsilon}$. \\

Since $\cV_{\epsilon}$ is $2\varepsilon$-packing, $\|\hat{\theta}(A; \pi) - V\|_2 > \varepsilon$ whenever $\hat{V} \neq V$. Consequently, $\Pr\left( \|\hat{\theta}(A; \pi) - \theta^*\|_2^2 > \varepsilon^2/2  \right) = \Pr\left( \|\hat{\theta}(A; \pi) - V\|_2^2 > \varepsilon^2/2  \right) \geq \Pr(\hat{V} \neq V)$. Since $V \to A \to \hat{V}$ from a Markov chain, using Fano's inequality~\cite{ThomasAndCover}, we can conclude that
\begin{align*}
    \Pr(\hat{V} \neq V) \geq 1 - \frac{I(V; A) + \log 2}{|\cV_{\varepsilon}|},
\end{align*}
where $I(V;A)$ denotes the mutual information between $V$ and $A$. Note that the event $\{\|\hat{\theta}(A; \pi) - \theta^*\|_2^2 > \varepsilon^2/2\}$ implies that no estimator based on the actions of $\pi$ can estimate $\theta^*$ within an error of $\varepsilon^2/2$ implying that $\pi$ incurs a cumulative regret of at least $\varepsilon^2MT/4$. Hence, to ensure a cumulative regret of $R_{\pi}(T)$ with probability at least $2/3$, we need to ensure $\Pr(\hat{V} \neq V) \leq \Pr\left( \|\hat{\theta}(A; \pi) - \theta^*\|_2^2 > \varepsilon^2/2  \right) < 1/3$ for $\varepsilon = 2\sqrt{R_{\pi}(T)/MT}$. Equivalently, $I(V;A) \geq 2\log|\cV_{\varepsilon}|/3 - \log 2$ with $\varepsilon = 2\sqrt{R_{\pi}(T)/MT}$. Using the standard bounds on packing numbers~\citep{Ball1997VolumeBounds} and noting that $R_{\pi}(T)$ is sub-linear in both $M$ and $T$, we can conclude that $I(V;A) \geq \Omega(d \log(MT))$. The bound on the communication cost is obtained using the data processing inequality. Let $Z$ denote all the messages broadcast by the server. Then $I(V;Z)$ is a lower bound on the downlink communication cost. Notice that $Z$ obeys the Markov chain $V \to Z \to A \to \hat{V}$ since the actions taken by the agents change with $V$ through the messages broadcast by the server. \\

From data processing inequality, we have, $I(V;Z) \geq I(V;A)$. In other words, since the messages $Z$ transfer the information about the actions of other agents to any given agent, they should at least have as much information about $\theta^*$ (or equivalently $V$) as much as $A$, the set of all actions, does. Combining this with the bound on $I(V;A)$, we arrive the required lower bound on the downlink communication cost.

\subsubsection{Proof for uplink cost}

We establish bounds on the uplink cost by drawing parallels to the distributed mean estimation problem. Consider the problem of distributed mean estimation where $X$ denotes the random variable corresponding to the observations by the agents and $Y$ denotes the messages sent by the agents to the server. Based on the messages $Y$, if no estimator can recover $\theta^*$ to within a mean squared error of $\varepsilon^2$, then no linear bandit algorithm can achieve a simple regret of $\varepsilon^2$. This is similar to the argument shown for the downlink case. This implies that, only if $\theta^*$ can be estimated to within an accuracy of $\varepsilon^2$, can there exist a linear bandit algorithm with a simple regret $\varepsilon^2$ and hence potentially with a cumulative regret of $2\varepsilon^2 MT$. Thus, we consider the problem of distributed mean estimation and show that unless all agents send $\Omega(d \log T)$ bits of information to the server, no estimator can estimate $\theta^*$ within an accuracy of $1/MT$ with probability at least $2/3$. For this case, we only consider the optimal rates instead of any sublinear rates. \\

The proof borrows ideas and techniques from the classical results in~\cite{Duchi2014DistributedEstimation}. In particular, the proof is similar to those of Proposition 2 and Theorem 1 in~\cite{Duchi2014DistributedEstimation}. However, it is different in its own sense as it builds upon those results and strengthens the lower bounds with the missing logarithmic factors. Thus, this proof might be of independent interest to the larger community. \\

The basic idea in the proof is to construct a Markov chain $\bfV \to X \to Y$, where $X$ represents the collection of all the observations at all the agents and $Y$ denotes the messages sent by the agents. In particular, $X = (X^{(1)}, X^{(2)}, \dots, X^{(M)})$ and $Y = (Y_1, Y_2, \dots, Y_M)$, where $X^{(j)}$ denotes the set of observations at agent $j$ and $Y_j$ denotes the messages sent by agent $j$ to the server. Since the agents cannot communicate with each other directly, we assume $Y_j$ depends only on $X^{(j)}$. \\

We begin with construction the hard instance for the random variable $\bfV$. Let $\cV = \{\pm 1\}^{d-1}$ and $\bfupsilon = (\upsilon_1, \upsilon_2, \dots, \upsilon_{\ell}) \in \cV^{\ell}$ for some $\ell \in \N$ specified later. For each $\bfupsilon \in \cV^{\ell}$, we define a vector $\mu_{\bfupsilon} \in R^{d-1}$ given by 
\begin{align*}
    \mu_{\bfupsilon} := \sum_{r = 1}^{\ell} \frac{2^{-r}}{\sqrt{d-1}} \upsilon_r.
\end{align*}
Note that $\|\mu_{\bfupsilon}\|_2 \leq \sum_{r = 1}^{\ell} \frac{2^{-r}}{\sqrt{d-1}} \|\upsilon_r\| \leq  \sum_{r = 1}^{\ell} 2^{-r} \leq 1$. Thus, $\mu_{\bfupsilon} \in \cB_{d-1}(1)$ for all $\bfupsilon \in \cV^{\ell}$. We also define a lifting map $f: \cB_{d-1}(1) \to \cS^{d}_{\geq 0}$, where $\cS^{d}_{\geq 0}$ denotes the hemisphere where the last coordinate only takes on non-negative values. It maps a vector $x \in \cB_{d-1}(1)$ to a vector $x' \in \cS^d_{\geq 0}$, where $x'_{1:d-1} = x$ and $x'_d = \sqrt{1 - \|x\|^2}$. In other words, $f$ lifts a point in unit hypersphere in $d - 1$ to the corresponding point on the unit hyper(hemi)sphere in $\R^d$ by adding the last component. From the definition, one can note that $f$ is a bijection. Furthermore, one can note that for any two points $x, x' \in \cB_{d-1}$, we have $\| x - x'\|_2 \leq \|f(x) - f(x')\|_2 \leq \sqrt{2} \| x - x'\|_2$. \\

Let $\bfV$ be a random variable drawn from the set $\cV^{\ell}$. We construct a linear bandit instance with $\theta^* = f(\delta\mu_{\bfV})$ for some $\delta \in (0,1)$ whose value is specified later. Since, $f$ is a bijection that preserves distances upto a constant, it equivalent to consider the problem of estimating $\theta^* = \theta_{\bfV} = \delta \mu_{\bfV}$, where the operation $f$ is assumed to have been carried out implicitly. Hence, for the remainder of the proof, we focus only of the problem on estimating $\theta_{\bfV}$. \\

To specify the observation model, we first need to set up some notations and definitions. Given a $u \in \{-1, 1\}$ and $\rho \in  [0,1]$, we define $P_{u,\rho}$ to be the distribution that assigns a mass of $(1 + \rho)/2$ to $u$ and $(1 - \rho)/2$ to $-u$. We overload the definition of $P_{u,\rho}$ for when $u = (u_1, u_2, \dots, u_p)^{\top} \in \{-1, 1\}^p$ is a vector. In this case, a sample from $P_{u, \rho}$ is a vector whose $i^{\text{th}}$ coordinate is drawn according to $P_{u_i, \rho}$, independently of others. For a given value of $\bfV = \bfupsilon$, each agent $j$, independent of other agents, receives an collection of $T$ i.i.d. samples, denoted by $X^{(j)} = \{X^{(j, k)}\}_{k = 1}^n$ for $j = 1,2, \dots, M$. Each sample $X^{(j,k)}$ is obtained by first drawing $\tilde{V}^{(j,k)}(\bfupsilon) = (\tilde{V}^{(j,k,1)}(\bfupsilon), \dots, \tilde{V}^{(j,k,\ell)}(\bfupsilon)) \in \cV^{\ell}$, where $\tilde{V}^{(j,k,r)}(\bfupsilon) \sim P_{\upsilon_r, \rho_0}$ for $r = 1, 2,\dots, \ell$ and then setting $X^{(j,k)} = \mu_{\tilde{V}^{(j,k)}(\bfupsilon)}$. In the above definition $\rho_0 := \delta/{T\ell}$. If $X_i^{(j)}$ denotes the vector obtained by taking the $i^{\text{th}}$ coordinate of all the $T$ samples at agent $j$, then from the above definition, one can note that $X_i^{(j)}$'s are independent across $i$, that is, each coordinate is independent of others.  \\

Having specified the underlying mean vector and the observation model, the next step is to use strong data processing inequality to quantitatively relate $I(\bfV; Y_j)$ and $I(X^{(i)}; Y_j)$. To establish this relation, we make use of Lemma 5 from~\cite{Duchi2014DistributedEstimation}. Before invoking the Lemma, we first need to establish a bound on the likelihood ratio of any realization $x_i$, of the random variable $X_i^{(j)}$, under two different values of $\bfV$, say $\bfupsilon, \bfupsilon'$ and also specify the measurable sets over which the bound holds. Throughout this part, we carry out all the analysis for a fixed agent $j$. \\

Note that any realization $x_i$ can be mapped to the realizations $\{\tilde{v}^{(j,k)}_i\}_{k=1}^T$, where $\tilde{v}^{(j,k)}_i = (\tilde{v}^{(j,k,1)}_i, \dots, \tilde{v}^{(j,k,\ell)}_i)$. For any fixed value of $r \in \{1,2,\dots, \ell\}$, consider the set $\tilde{v}^{(j,1:T,r)}_i = \{ \tilde{v}^{(j,1:T,r)}_i \}_{k = 1}^T$ which only contains values from $\{-1, 1\}$ drawn from $P_{\upsilon_{r, i}, rho_0}$. Here $\upsilon_{r, i}$ denotes the $i^{\text{th}}$ coordinate of $\upsilon_r$. Suppose this collection contains $T_1$ instances of $+1$ and $T - T_1$ of $-1$, then the likelihood ratio any under any pair $(\bfupsilon, \bfupsilon')$ can be bounded as $\left(\frac{1+\rho_0}{1 - \rho_0}\right)^{|T - 2T_1|}$.  If $S_r$ denotes the absolute value of sum of the elements in $\tilde{v}^{(j,1:T,r)}_i$, then this bound on the likelihood ratio can be rewritten as $\left(\frac{1+\rho_0}{1 - \rho_0}\right)^{S_r}$. Since all $\ell$ sequences in $\{\tilde{v}^{(j,1:T, r)}_i\}_{r=1}^{\ell}$ are independent of each other, the likelihood ratio can be bounded as
\begin{align*}
    \sup_{x_i} \sup_{\bfupsilon, \bfupsilon'\in \cV^{\ell}} \frac{\Pr(x_i|\bfupsilon)}{\Pr(x_i|\bfupsilon')} \leq \left( \frac{1 + \delta/{T\ell}}{1 - \delta/{T\ell}} \right)^{S_1} \cdots \left( \frac{1 + \delta/{T\ell}}{1 - \delta/{T\ell}} \right)^{S_{\ell}} = \left( \frac{1 + \delta/{T\ell}}{1 - \delta/{T\ell}} \right)^{S_1 + \cdots + S_{\ell}} .
\end{align*}

To construct a measurable set $G_i$ corresponding to each coordinate $i$ on which the likelihood ratio can be appropriately bounded, consider the set $\cZ$ defined as
\begin{align*}
    \cZ = \left\{(v_1, \dots, v_{T}) \in \{-1, +1\}^{T} : \left|\sum_{r = 1}^{T} v_r \right| \leq a\sqrt{2T} + \delta/\ell \right\}
\end{align*}
for some $a > 0$ to be determined later. We set $G_i = \{x_i : \tilde{v}^{(j,1:T,r)}_i  \in \cZ \ \forall \ r \in \{1, 2, \dots, \ell\}\}$ for all coordinates $i$. That is, $G_i$ consists of all sequences $x_i$ such that all its corresponding $\tilde{v}$ sequences belong to $\cZ$. When $X_i^{(j)}$ belongs to the $\sigma$-field generated by the elements of $G_i$, we can bound the likelihood ratio as 
\begin{align*}
    \sup_{x_i \in \sigma(G_i)} \sup_{\bfupsilon, \bfupsilon'\in \cV^{\ell}} \frac{\Pr(x_i|\bfupsilon)}{\Pr(x_i|\bfupsilon')} & \leq \left( \frac{1 + \delta/T\ell}{1 - \delta/T\ell} \right)^{\ell(a \sqrt{2T}  + \delta/\ell)}  \\
    & \leq \exp\left( 3\ell(a \sqrt{2T} + \delta/\ell)  \cdot \frac{\delta}{T\ell}\right) \\
    & \leq \exp\left( 3(a + \delta/\ell)\delta\sqrt{\frac{2}{T}}\right) := \exp(\varphi).
\end{align*}
for all $T \geq 4\ell/3$. As the last step to invoke the Lemma, we define $E_i^{(j, r)} = \1\{\tilde{v}^{(j,1:T,r)}_i  \in \cZ\}$ and $E_i^{(j)} = \prod_{r = 1}^{\ell} E_i^{(j, r)}$, where $\1\{A\}$ denotes the indicator variable for the event $A$. Using the definition of $G_i$, we can also define $E_i^{(j)}$ as $\1\{X_i^{(j)} \in G_i\}$. Lastly, we also define $E^{(j)} = \prod_{i = 1}^{d-1} E_i^{(j)}$.  \\

We are now all set to invoke Lemma 5 from~\cite{Duchi2014DistributedEstimation}. Using the Lemma, we can conclude that for the message sent by agent $j$, $Y_j$, the following inequality holds
\begin{align*}
    I(\bfV; Y_j) \leq 2(e^{4\varphi} - 1)^2 I(X^{(j)}; Y_j| E^{(j)} = 1) + \sum_{i = 1}^{d-1} H(E_i^{(j)}) + \sum_{i = 1}^{d-1} H(\bfV_i)\Pr(E_i^{(j)} = 0),
\end{align*}
where $H(W)$ denotes the entropy of the random variable $W$ and $\bfV_i$ refers to the tuple formed by taking the $i^{\text{th}}$ coordinate of all the elements in tuple represented by $\bfV$. \\

For the first term, note that for $T \geq 200(a + 1)^2$, $\varphi \leq 5/16$ implying that $2(\exp(4\varphi) - 1)^2 \leq 2 (8 \varphi)^2 = 128\varphi^2$. Using the standard concentration bounds for Bernoulli random variables, we can conclude that $\Pr(E_i^{(j, r)} = 0) \leq 2\exp(-a^2)$ and consequently $\Pr(E_i^{(j)} = 0) \leq 2\ell\exp(-a^2)$. On plugging these results into the previous equation, we obtain,
\begin{align*}
    I(\bfV; Y_j)  & \leq \frac{2304(a + \delta/\ell)^2 \delta^2}{T} I(X^{(j)}; Y_j| E^{(j)} = 1) + \sum_{i = 1}^{d-1} \sum_{r = 1}^{\ell} H(E_i^{(j, r)}) + 2\sum_{i = 1}^{d-1} \ell^2\exp(-a^2) \\
    & \leq \frac{2304(a + \delta/\ell)^2 \delta^2}{T} \sum_{k = 1}^{T} I(X^{(j, k)}; Y_j| E^{(j)} = 1, X^{(j, 1:(k-1))})  +(d-1)\ell (h_2(2e^{-a^2}) +  2\ell e^{-a^2}) \\
    & \leq \frac{2304(a + \delta/\ell)^2 \delta^2}{T} \sum_{k = 1}^{T} \min\{ H(X^{(j, k)}| E^{(j)} = 1, X^{(j, 1:(k-1))}), H(Y_j| E^{(j)} = 1, X^{(j, 1:(k-1))}) \} \\
    & ~~~~~~~~~~~~~~~~~~~~~~~~~~~~~~~~~~~~~~~~ +(d-1)\ell (h_2(2e^{-a^2}) +  2\ell e^{-a^2}) \\
    & \leq \frac{2304(a + \delta/\ell)^2 \delta^2}{T} \sum_{k = 1}^{T} \min\{ H(X^{(j, k)}), H(Y_j) \}  +(d-1)\ell (h_2(2e^{-a^2}) +  2\ell e^{-a^2}) \\
    & \leq \frac{2304(a + \delta/\ell)^2 \delta^2}{T} \sum_{k = 1}^{T} \min\{ (d-1)\ell, H(Y_j) \}  +(d-1)\ell (h_2(2e^{-a^2}) +  2\ell e^{-a^2})) \\
    & \leq 2304(a + \delta/\ell)^2 \delta^2 \min\{ d\ell, H(Y_j) \}  +(d-1)\ell (h_2(2e^{-a^2}) +  2\ell e^{-a^2}),
\end{align*}
where we used the relation $I(W;W') \leq \min\{H(W), H(W')\}$ for two random variables along with the fact that conditioning reduces entropy. In the above equations $h_2(p) := -p \log_2(p) - (1-p)\log_2(1-p)$, denotes the entropy of a Bernoulli random variable with mean $p$, for $p \in [0,1]$. Since each message $Y_j$ depends only on $X^{(j)}$, we have,
\begin{align*}
    I(\bfV; Y) & \leq \sum_{j = 1}^M I(\bfV; Y_j) \\
    & \leq \sum_{j = 1}^M \left[ 2304(a + \delta/\ell)^2 \delta^2 \min\{ (d-1)\ell, H(Y_j) \}  +(d-1)\ell (h_2(e^{-a^2}) +  \ell e^{-a^2})\right].
\end{align*}
This gives us the bound on $I(\bfV; Y)$ in terms of $H(Y_j)$, which is a lower bound on the required communication cost to send the message corresponding to agent $j$. The last step is use Fano's inequality to translate this bound to a bound on the estimation error. \\

Let $\hat{\theta}(Y)$ denote the estimate obtained at the server using the received messages $Y$. Similar to the previous case, let $\hat{\bfV} := \argmin_{\bfupsilon \in \cV^{\ell}} \|\hat{\theta}(Y) - \theta_{\bfupsilon}\|_2$ denote the corresponding estimate of $\bfV$. For $\hat{\bfV} \neq \bfV$, the estimation error $\|\hat{\theta}(Y) - \theta_{\bfupsilon}\|_2^2$ satisfies $\|\hat{\theta}(Y) - \theta_{\bfupsilon}\|_2^2 \geq \|\theta_{\hat{\bfV}} - \theta_{\bfV} \|_2^2/4 \geq \delta^2 \|\mu_{\hat{\bfV}} - \mu_{\bfV} \|_2^2 \geq \delta^2 4^{-\ell-1}/(d-1)$. The last bound follows by noting that $\|\mu_{\bfupsilon} - \mu_{\bfupsilon'}\| \geq \delta 2^{-\ell}/\sqrt{d-1}$ for $\bfupsilon \neq \bfupsilon'$.
We set $a := 2 \sqrt{\log(4M\ell)}$, $\delta^2 := (2304M(a + 1/\ell))^{-1} $ and $\ell = \log_4(T)$. On plugging in these values, we can conclude that whenever $\hat{\bfV} \neq \bfV$, the estimation error is at least $(4(d-1)MT)^{-1}$. On the other hand, since $\bfV \to Y \to \hat{\bfV}$ forms a Markov chain, Fano's inequality tells us that
\begin{align*}
    \Pr(\hat{\bfV} \neq \bfV) & \geq 1 - \frac{I(\bfV; Y) + \log 2}{|\cV^{\ell}|} \\
    & \geq 1 - \frac{1}{(d-1)\ell} \left\{ \sum_{j = 1}^M \left[ \frac{1}{4M} \min\{ (d-1)\ell, H(Y_j) \}  +(d-1)\ell (h_2(e^{-a^2}) +  \ell e^{-a^2})\right] + \log 2 \right\} \\
    & \geq 1 - \left\{ \sum_{j = 1}^M \left[ \frac{1}{4M} \min\left\{ 1, \frac{H(Y_j)}{(d-1)\ell} \right\}  +\frac{6}{80M^2 \ell^2} +  \frac{1}{256M^4\ell^3}\right] + \frac{\log 2}{(d-1)\ell} \right\}.
\end{align*}
In the last step, we used the value of $a$ along with the fact that $h_2(p) \leq 1.2\sqrt{p}$ for all $p \in [0,1]$. For $T \geq 2048$, we have that
\begin{align*}
    \Pr(\hat{\bfV} \neq \bfV) & \geq \frac{4}{5} - \frac{1}{4M} \sum_{j = 1}^M  \min\left\{ 1, \frac{H(Y_j)}{(d-1)\ell} \right\}.
\end{align*}
Thus, to ensure that the estimation error exceeds $\Omega(1/MT)$ with probability no more than $1/3$, we need to ensure $\frac{H(Y_j)}{(d-1)\ell}$ is $\Omega(1)$ for all agents $j$. Since $\ell$ is $\Omega(\log T)$, it implies that all agents needs to send at least $\Omega(d \log T)$ bits of information to ensure optimal regret.

\section{Sparse PLS}
\label{sec:Sparse_PLS_Analysis}

In this section, provide additional details and analysis of Sparse-PLS. We begin with the pseudo codes.

\begin{algorithm}[H]
	\caption{Sparse-PLS Norm Estimation:  Agent $j \in \{1,2,\dots, M\}$}
	\label{alg:norm_est_agent_sparse}
	\begin{algorithmic}[1]
        \STATE \textbf{Input}: The set of actions $\cB_s$
        \STATE Set $k \leftarrow 1$
        \WHILE{\texttt{True}}
            \STATE Play each vector in $\cB_s$ for $s_k$ times and compute the sample mean $\hat{\theta}^{(j)}_k$
            \STATE $\tilde{\theta}^{(j)}_k \leftarrow \textsc{Clip}(\hat{\theta}^{(j)}_k, R_k + B_k)$
            \STATE $Q(\tilde{\theta}^{(j)}_k) \leftarrow \textsc{StoQuant}(\tilde{\theta}^{(j)}_k, \alpha_k, R_k + B_k)$
            \STATE Send $Q(\tilde{\theta}^{(j)}_k)$ to the server
            \IF{received \texttt{terminate} from server}
            \STATE \textbf{break}
            \ELSE
            \STATE $k \leftarrow k + 1$
            \ENDIF
        \ENDWHILE
	\end{algorithmic}
\end{algorithm}

\begin{algorithm}[H]
	\caption{Sparse-PLS Norm Estimation: The Server}
	\label{alg:norm_est_server_sparse}
	\begin{algorithmic}[1]
        \STATE \textbf{Input}: The set of actions $\cB_s$
        \STATE Set $k \leftarrow 1$
        \WHILE{\texttt{True}}
            \STATE Compute $\hat{\theta}_{k}^{(\textsc{serv})} := \argmin_{\theta} \frac{d}{m}\left\|\frac{1}{M} \sum_{j = 1}^M Q(\tilde{\theta}^{(j)}_k) - X\theta \right\|_2^2 + \lambda_k \|\theta\|_1$
            \IF{$\tau_k \leq \frac{1}{4}\|\hat{\theta}_{k}^{(\textsc{serv})}\| $}
            \STATE Server sends terminate to all agents
            \STATE \textbf{break}
            \ELSE
            \STATE $k \leftarrow k + 1$
            \ENDIF
        \ENDWHILE
	\end{algorithmic}
\end{algorithm}

In the pseudo codes for Sparse-PLS, $X$ refers to the $\R^{m \times d}$ matrix obtained by stacking the vectors in $\cB_s$ one under the other. The parameters for Sparse-PLS are set as $s_k := \lceil 40 \sigma^2 d \log(16MK/\delta) 4^k \rceil$, $t_k := \lceil m M s_k^2/d \rceil$, $R_k := 2^{-k}\sqrt{m/d}$, $B_k = 7\tau_k \sqrt{m/d}$, $\tau_k := 3 \cdot 2^{(-(k+1)}/\sqrt{M}$, $\alpha_k = \alpha_0 \sigma \sqrt{s/s_k}$, $\beta_k = \beta_0 \tau_k$ and $\lambda_k := 4\sigma \sqrt{\frac{3}{2ms_k}} (\sqrt{\log(2d)} + \sqrt{\log(4/\delta)}$. The underlying philosophy behind the choice of these parameters is similar to that of PLS. The only additional parameters in Sparse-PLS is the regularization constant $\lambda_k$ and the size of the set $\cB_s$, $m$. $\lambda_k$ is chosen based on analysis of the LASSO estimator~\cite{Rigollet2017} while $m$ is chosen to ensure that $X$ satisfies the restricted eigenvalue condition~\cite{Bickel2009}. In particular, based on the result in~\cite{Baraniuk2008}, we set $m = 80(s \log(150d/s) + \log(4/\delta))$ which ensures that with probability at least $1 - \delta/2$ over the randomness of $\cB_s$ the following holds for any $\theta \in \cC_s$, 
\begin{align*}
    \frac{3}{4} \|\theta\|_2 \leq \sqrt{\frac{d}{m}}\|X\theta\|_2 \leq \frac{5}{4} \|\theta\|_2.
\end{align*}
In the above definition $\cC_s = \{\theta: \|\theta_{S^c}\|_1 \leq 3 \|\theta_S\|_1  \text{ for all } S \subset \{1,2,\dots, d\} \text{ with } |S| \leq s\}$ where $\theta_S$ refers to the sub-vector of $\theta$ corresponding to the coordinates indexed by $S$. \\

\begin{algorithm}
    \caption{Sparse-PLS Refinement:  Agent $j \in \{1,2,\dots, M\}$}
    \label{alg:ref_est_agent_sparse}
    \begin{algorithmic}[1]
            \STATE \textbf{Input}: The epoch index at the end of Norm Estimation stored as $k_0$, The set of actions $\cB_s$
            \STATE $\bar{\theta}_{k_0 - 1} \leftarrow 0, k \leftarrow k_0$
            \WHILE{time horizon $T$ is not reached}
                \STATE Play each vector in $\cB_s$ for $s_k$ times and compute the sample mean $\hat{\theta}^{(j)}_k$
                \STATE $\tilde{\theta}^{(j)}_k \leftarrow \textsc{Clip}(\hat{\theta}^{(j)}_k - \bar{\theta}_{k - 1}, R_k + B_k)$
                \STATE $Q(\tilde{\theta}^{(j)}_k) \leftarrow \textsc{StoQuant}(\tilde{\theta}^{(j)}_k, \alpha_k, R_k + B_k)$
                \STATE Send $Q(\tilde{\theta}^{(j)}_k)$ to the server
                \STATE Receive $Q(\hat{\theta}_{k}^{(\textsc{serv})})$ from the server
                \STATE $\bar{\theta}_{k} \leftarrow \bar{\theta}_{k - 1}  + Q(\hat{\theta}_{k}^{(\textsc{serv})})$
                \IF{$k = k_0$}
                \STATE Set $\mu_0 \leftarrow \|\bar{\theta}_{k}\|_2$
                \ENDIF
                \STATE Play the action $a = \bar{\theta}_{k}/\|\bar{\theta}_{k}\|$ for the next $t_{k}$ rounds.
                \STATE $k \leftarrow k + 1$
            \ENDWHILE
    \end{algorithmic}
\end{algorithm}

\begin{algorithm}
    \caption{Sparse-PLS Refinement: The Server}
    \label{alg:ref_est_server_sparse}
    \begin{algorithmic}[1]
            \STATE \textbf{Input}: The epoch index at the end of Norm Estimation stored as $k_0$, The set of actions $\cB_s$
            \STATE $\bar{\theta}_{k_0 - 1} \leftarrow 0, k \leftarrow k_0$
            \WHILE{time horizon $T$ is not reached}
                \STATE Receive $Q(\tilde{\theta}^{(j)}_k)$ from all the agents
                \STATE Compute $\hat{\theta}_{k}^{(\textsc{serv})} = \bar{\theta}_{k - 1} + \argmin_{\theta}\frac{d}{m}\left\|\frac{1}{M} \sum_{j = 1}^M Q(\tilde{\theta}^{(j)}_k) - X(\theta - \bar{\theta}_{k-1}) \right\|_2^2 + \lambda_k \|\theta\|_1$
                \STATE $Q(\hat{\theta}_{k}^{(\textsc{serv})}) \leftarrow \textsc{DetQuant}(\hat{\theta}_{k}^{(\textsc{serv})} - \bar{\theta}_{k - 1}, \beta_k, B_k + \tau_{k})$ and broadcasts it to all agents
                \STATE $k \leftarrow k + 1$
            \ENDWHILE
    \end{algorithmic}
\end{algorithm}

Before the proof of Theorem~\ref{theorem:sparse_bandits}, we state and prove a lemma analogous to Lemma~\ref{lemma:theta_estimate_error} for sparse bandits.

\begin{lemma}
    For any epoch $k$ in Sparse-PLS, the estimate $\hat{\theta}_{k}^{(\textsc{serv})}$ satisfies
    \begin{align*}
        \|\hat{\theta}_{k}^{(\textsc{serv})} - \theta^*\| \leq \tau_k,
    \end{align*}
    with probability at least $1 - \delta/2$.
    \label{lemma:sparse_theta_error}
\end{lemma}

\begin{proof}

Similar to the proof of Lemma~\ref{lemma:theta_estimate_error}, we use induction and also define $\bar{\theta}_{k-1} := 0$ for all epochs $k$ during the norm estimation stage. In Sparse-PLS, the estimate $\hat{\theta}_{k}^{(\textsc{serv})}$ is obtained by minimizing
\begin{align*}
    \cL_k(\theta) & := \frac{d}{m}\left\|\frac{1}{M} \sum_{j = 1}^M Q(\tilde{\theta}^{(j)}_k) - X(\theta - \bar{\theta}_{k-1}) \right\|_2^2 + \lambda_k \|\theta\|_1 \\
    & = \frac{d}{m}\left\|\frac{1}{M} \sum_{j = 1}^M (\tilde{\theta}^{(j)}_k + \eta_k^{(j)}) - X(\theta - \bar{\theta}_{k-1}) \right\|_2^2 + \lambda_k \|\theta\|_1 \\
    & = \frac{d}{m}\left\|\frac{1}{M} \sum_{j = 1}^M(\hat{\theta}^{(j)}_k - X\bar{\theta}_{k-1})\1_{R_k + B_k} + \frac{1}{M} \sum_{j = 1}^M \eta_k^{(j)} - X(\theta - \bar{\theta}_{k-1}) \right\|_2^2 + \lambda_k \|\theta\|_1 \\
    & = \frac{d}{m}\left\|\frac{1}{M} \sum_{j = 1}^M(\hat{\theta}^{(j)}_k - X\bar{\theta}_{k-1})\1_{R_k + B_k} + \frac{1}{M} \sum_{j = 1}^M \eta_k^{(j)} - X(\theta - \bar{\theta}_{k-1}) \right\|_2^2 + \lambda_k \|\theta\|_1 \\
    & = \frac{d}{m}\left\|X\theta^* + \frac{1}{M} \sum_{j = 1}^M\Delta_k - X\bar{\theta}_{k-1} + \frac{1}{M} \sum_{j = 1}^M \eta_k^{(j)} - X(\theta - \bar{\theta}_{k-1}) \right\|_2^2 + \lambda_k \|\theta\|_1 \\
    & = \frac{d}{m}\left\|X\theta^* + \frac{1}{M} \sum_{j = 1}^M\Delta_k  + \frac{1}{M} \sum_{j = 1}^M \eta_k^{(j)} - X\theta \right\|_2^2 + \lambda_k \|\theta\|_1,
\end{align*}
where $\Delta_k$ corresponds to clipped sub-Gaussian observation noise. Using the result of Theorem 2.18 in ~\cite{Rigollet2017}, we can conclude that
\begin{align*}
    \|\hat{\theta}_{k}^{(\textsc{serv})} - \theta^*\|_2^2 \leq \frac{1024sd \log(8K/\delta)}{mM} \left( \frac{\sigma^2}{s_k} + \frac{\alpha_k^2}{4s} \right).
\end{align*}
Plugging in the choice of $s_k$ and $\alpha_k$ yields,
\begin{align*}
    \|\hat{\theta}_{k}^{(\textsc{serv})} - \theta^*\|_2 \leq \frac{3}{\sqrt{M}} \cdot 2^{-(k+1)} = \tau_k,
\end{align*}
with probability at least $1 - \delta/K$. Here, similar to proof of Lemma~\ref{lemma:theta_estimate_error} the choice of $R_k$ and $B_k$ allows us to use the clipped sub-Gaussian concentration which is implicitly used while invoking the result from~\cite{Rigollet2017}. \\

Using an argument similar to the one in Lemma~\ref{lemma:theta_estimate_error}, we can conclude that the above inequality holds for epochs during the algorithm with probability at least $1 - \delta/2$.

\end{proof}

\subsection{Proof of Theorem~\ref{theorem:sparse_bandits}}

The analysis for the regret performance of Sparse-PLS is almost identical to that of PLS, as described in Appendix~\ref{sec:PLS_analysis}. Once again, the regret is decomposed into the sum of regret incurred during the norm estimation stage and the refinement stage. \\

Recall that the regret during the norm estimation stage of PLS is bounded using the result in Lemma~\ref{lemma:theta_estimate_error}. Since Lemma~\ref{lemma:sparse_theta_error} is identical to Lemma~\ref{lemma:theta_estimate_error}, the proof of Lemma~\ref{lemma:regret_norm_est} follows through almost unchanged for the case of Sparse-PLS. The only difference in the case of Sparse-PLS is that there are $m$ actions in the basis set instead of $d$ as in the case of PLS. This scales the corresponding term in the regret incurred during the norm estimation stage by a factor of $m/d$ resulting in an overall regret of $\cO(\sqrt{sdMT\log(d/\delta)} \log(MT) \log(\log T/\delta) )$. \\

Similarly, the analysis in the refinement stage also follows through identically but for a couple of minor changes. The regret during a exploration sub-epoch is also scaled by a factor of $m/d$ for the same reason as in the case of the norm estimation stage. The analysis for exploration sub-epoch follows through as is with the different value of $t_k$ which is also scaled by a factor of $s/d$. Carrying out the same steps with these updated values result in an overall regret of $\cO(\sqrt{sdMT\log(d/\delta)} \log(MT) \log(\log T/\delta) )$, as required. \\

For the communication cost, the bound on the downlink cost follows as in case for PLS. For the case of the uplink cost, note that in Sparse-PLS, the transmitted vector lies in $\R^m$. Using the same argument in used on the proof of  Lemma~\ref{lemma:channel_capacity} for vectors in $\R^m$ along with the updated choice of $R_k$, $B_k$ and $\alpha_k$, we can conclude that each uplink message is at most $\cO(m)$ bits long, resulting in an overall uplink cost of $\cO(m\log T) = \cO(s \log dT)$.


\end{document}